\documentclass{article}

\usepackage{microtype}
\usepackage{graphicx}
\usepackage{caption, subcaption}
\usepackage{booktabs}

\usepackage{hyperref}
\usepackage{wrapfig, amsmath, amsthm, mathtools}
\usepackage{bbold}
\usepackage{amsfonts}
\usepackage{amssymb}
\usepackage{xcolor}

\usepackage{multirow}

\usepackage{makecell}
\usepackage{paralist}

\def\E{{\bf E}}

\def\P{{\bf P}}

\def\0{{\bf 0}}
\def\1{{\bf 1}}

\def\RB{{\mathbb R}}

\def\argmin{\mathop{\rm argmin}}

\newcommand{\Da}{\mathcal{U}}
\newcommand{\Var}{\text{Var}}

\newcommand{\Bias}{\text{B}}

\newtheorem*{remark*}{Remark}
\newtheorem{theorem}{Theorem}
\newtheorem*{theorem*}{Theorem}
\newtheorem{lemma}{Lemma}
\newtheorem*{lemma*}{Lemma}
\newtheorem{definition}{Definition}

\newtheorem{cor}{Corollary}
\newtheorem{assumption}{Assumption}
\newtheorem{example}{Example}
\numberwithin{theorem}{section}
\numberwithin{lemma}{section}
\numberwithin{remark}{section}
\numberwithin{cor}{section}

\newcommand{\lemref}[1]{Lemma~\ref{#1}}
\newcommand{\thmref}[1]{Theorem~\ref{#1}}

\usepackage[accepted]{icml2021}

\icmltitlerunning{Correcting Exposure Bias for Link Recommendation}

\begin{document}

\twocolumn[
\icmltitle{Correcting Exposure Bias for Link Recommendation}

\begin{icmlauthorlist}
\icmlauthor{Shantanu Gupta}{intern,cmu}
\icmlauthor{Hao Wang}{rutgers}
\icmlauthor{Zachary C. Lipton}{cmu}
\icmlauthor{Yuyang Wang}{amazon}
\end{icmlauthorlist}

\icmlaffiliation{intern}{Work done while interning at Amazon}
\icmlaffiliation{cmu}{Carnegie Mellon University}
\icmlaffiliation{rutgers}{Rutgers University}
\icmlaffiliation{amazon}{Amazon Web Services (AWS) AI Labs}

\icmlcorrespondingauthor{Shantanu Gupta}{shantang@cs.cmu.edu}

\vskip 0.3in
]

\printAffiliationsAndNotice{}

\begin{abstract}
Link prediction methods are frequently 
applied in recommender systems, e.g., 
to suggest citations for academic papers
or friends in social networks. 
However, exposure bias can arise
when users are systematically
underexposed to certain relevant items.
For example, in citation networks, 
authors might be more likely 
to encounter 
papers 
from their own field and thus 
cite them preferentially.
This bias can propagate through 
naively trained link predictors,
leading to both biased evaluation
and high generalization error
(as assessed by true relevance).
Moreover, this bias can 
be exacerbated by feedback loops.
We propose estimators that leverage
known exposure probabilities
to mitigate this bias 
and consequent feedback loops.
Next, we provide a loss function 
for \emph{learning} the exposure probabilities from data.
Finally, experiments on semi-synthetic data 
based on real-world citation networks,
show that our methods reliably identify 
(truly) relevant citations.
Additionally, our methods lead to 
greater diversity in the recommended papers' fields of study.
The code is available at 
\url{github.com/shantanu95/exposure-bias-link-rec}.
\end{abstract}

\section{Introduction}
\label{sec:introduction}
Diverse application domains, 
including both citation networks 
and social networks,
are characterized by graph-structured data.
Here, nodes represent entities 
(like papers or users)
and edges represent associations between two nodes
(like citations, friendships, or follows). 
Link recommender systems (RSs) 
leverage node attributes 
and existing links
to suggest new nodes 
that a given node 
\emph{should} link to
\citep{li2017survey, bai2019scientific, ma2020review}. 
Typically, RSs are trained and evaluated 
directly on the observed graph, 
raising concerns about exposure bias---many
missing links are false negatives,
and did not form due to lack of exposure
rather than a lack of affinity.

Consider the example of an RS 
that recommends relevant citations to authors 
given attributes of their paper
(like title, abstract, etc.). 
In this case, equally relevant papers 
from different fields of study (FOS)
might be less cited historically
because authors have been preferentially exposed 
to papers in their own FOS.
In the \textit{observed} citation graph, 
a number of relevant papers 
are observed as \textit{not cited}
because the user was not exposed to those papers.
Thus evaluating a link RS directly on the observed graph 
may yield a biased estimate of the true risk.

Exposure bias can exacerbate popularity bias,
causing relevant but unpopular items 
to not be shown \citep{chen2020bias}.
In social networks, diverse recommendations 
can help users form links with communities 
they would otherwise not discover 
\citep{li2017survey, brandao2013using}.
In citation networks, exposure bias 
can also lead to lines of research 
being duplicated across fields.
Examples include \textit{model-based science} 
and \textit{linear canonical transforms},
which were developed in isolation
\citep{vincenot2018new, liberman2015independent}.
Thus it would be valuable to have an RS
that recommends relevant low-exposure nodes.

In this paper, we call the probability that a node
is exposed to another node 
the \textit{propensity score};
and we call the probability 
that, given exposure, 
a node links to another node 
the \textit{link probability}.
An RS trained directly on the observed data
will underestimate the link probability
for low propensity nodes
relative to high propensity nodes.
We demonstrate this with a simple example 
in the context of academic
citation recommendation.
\begin{example}[\textbf{Exposure Bias}]
Let's say that there are two FOS: 
Machine Learning (ML) and Physics (PH), 
with $n$ papers in each. 
An ML researcher is looking for papers to cite. 
The probability of them being exposed to papers
in ML and PH is $0.9$ and $0.6$, respectively. 
Given exposure, 
the probabilities that they cite papers
from ML and PH are $0.8$ and $0.8$, respectively. 
In the observed data, we will see 
$\approx 0.72n (= 0.9 \times 0.8 n)$ ML papers cited 
and $\approx 0.48n (= 0.6 \times 0.8 n)$ PH papers cited. 
Thus, if we directly learn link probabilities 
from the observed data,
the probability of citing a PH paper
will be underestimated 
($0.48$ instead of $0.8$) 
more than that of an ML paper 
($0.72$ instead of $0.8$).
This shows that equally relevant papers
with lower propensity may be deemed less relevant.
\end{example}

To begin, we show that evaluating an RS 
naively on the observed data 
provides a misleading measure 
of its risk.
Instead, we argue that an RS 
should be evaluated via the risk
that would have been incurred
had every user been exposed to every node.
We call this the \textit{true risk}.
We propose three estimators of the true risk 
that use known propensity scores for estimation 
(Section~\ref{sec:estimating-risk}).
The key idea is to weight the positive and negative links 
using functions of the propensity scores and link probabilities.
Each of the three estimators 
uses a different weighting scheme.
We provide sufficient conditions 
for when they will have lower bias 
than the naive estimator for the true risk.
We then derive a generalization bound 
that shows that, with high probability, 
the true risk is close to the risk 
estimated by our proposed methods.
We use this bound to motivate a loss function 
that can be used to simultaneously \textit{learn} 
the link probabilities and propensity scores 
(Section~\ref{sec:learning-estimators}).
Next, under a simplified model of link recommendation,
where nodes belong to one of a finite number of categories,
we prove that feedback loops arise under exposure bias
and that they worsen at a faster rate 
for lower propensity nodes 
(Section~\ref{sec:feedback-loops}). 
We further show that 
accounting for exposure bias 
can help alleviate them.

We empirically validate our methods 
on real-world citation data 
from the Microsoft Academic Graph (MAG) \citep{sinha2015overview} (Section~\ref{sec:experiments}).
Since true exposure values 
are not available in the real data,
we construct a semi-synthetic data 
with simulated exposure and link probabilities.
Our methods lead to higher precision and recall 
against true citations than the naive method.
On real data, our methods maintain comparable performance 
to the naive method on metrics 
computed against the observed data 
and recommend more papers from different fields-of-study.

\section{Related Work}
\label{sec:related}
There is a rich literature
for correcting bias in RSs. 
\citet{swaminathan2015counterfactual} present 
a counterfactual risk minimization framework 
for learning from logged bandit feedback. 
\citet{joachims2017unbiased} use a counterfactual inference framework
to counteract
selection bias in click data.
\citet{schnabel2016recommendations} propose 
unbiased performance estimators for RSs 
that use known propensity scores 
when explicit item ratings are observed with selection bias.
\citet{ma2019missing} recover propensities 
under the low nuclear norm assumption. 
\citet{wang2020causal} use exposure data 
to construct a substitute for unobserved confounders.
\citet{wang2021fairness} show that bandit algorithms 
can lead to an unfair allocation of exposure across arms,
and to overcome this issue, 
they propose an alternative formulation,
where each arm receives exposure proportional to its merit.
The implicit feedback setting, 
where user interactions, 
such as clicking and listening 
(as opposed to explicit ratings), 
are used to train the RS, 
is more closely related to our setting. 
It is known that in this setting, 
some negative examples are false negatives 
due to exposure bias \citep[Section~4.1]{jeunen2019revisiting}.
\citet{yang2018unbiased} use inverse propensity scoring
to create an unbiased evaluator
for this setting using inverse-propensity-scoring
based methods. 
\citet{liang2016modeling} model exposure 
as a latent variable and incorporate it 
into a collaborative-filtering approach.
\citet{liang2016causal} use exposure and click models 
to re-weight samples to make unbiased predictions.
Our work leverages ideas from these works, 
especially the approach 
of re-weighting samples to counter the bias. 
However, this work addresses the item recommendation regime 
and the methods do not translate to the link prediction setting.

\citet{chang2009relational} develop a relational topic model 
for link prediction in document graphs. 
\citet{wang2017relational} extend this work 
by incorporating deep learning under the framework of Bayesian deep learning~\cite{CDL,BDL,BDLSurvey}. 
In social networks, learning-based methods 
and proximity-based methods are leveraged \citep{OEM,li2017survey}.
\citet{masrour2020bursting} study filter bubbles in link prediction 
and propose a method to recommend more diverse links.
Citation recommenders use paper data and metadata 
for training \citep{beel2016paper, ma2020review}. 
Some systems use local citation contexts 
to improve predictions \citep{wang2020deep, haruna2017context}.
In contrast, our goal
in this work 
is to augment existing models 
such that they account for exposure bias 
during both training and evaluation.

Addressing feedback loops in RSs,
\citet{chaney2018algorithmic} and \citet{mansoury2020feedback}
use simulations to demonstrate 
that they can arise,
amplifying popularity bias and user homogeneity. 
\citet{sun2019debiasing} present several 
matrix-factorization-based debiasing algorithms 
to prevent feedback loops.
\citet{sinha2016deconvolving} propose a method
to identify the items affected by feedback loops 
and recover the user's intrinsic preferences. 
\citet{jiang2019degenerate} show that feedback loops 
can create echo-chambers and filter bubbles.
\citet{zhao2017men} show that models 
amplify biases in training data 
and propose a constraint-based method to mitigate this.
In contemporaneous work, 
\citet{wang2021directional} extend this work 
and propose another metric for measuring bias 
and empirically show that it disentangles 
the direction of amplification.

\section{Estimating Risk under Exposure Bias}
\label{sec:estimating-risk}
\paragraph{Notation.}
Our dataset is a directed graph $\mathcal{G}(V, E)$,
where $V = \{v_1, \hdots, v_n\}$ is the set of $n$ nodes 
and $E$ is the set of edges,
s.t. $(i, j) \in E$ denotes a link from $v_i$ to $v_j$. 
We denote by $\Da = \{ (i, j) : i \in [n],\, j \in [n], \text{s.t.}\, i \neq j \}$
the possible (including missing) links in the graph;
by $\pi_{ij}$ the \textit{propensity}, i.e., 
the probability that $v_i$ is exposed to $v_j$;
and by
$y_{ij}$ the \textit{link probability}, i.e., 
the probability that $v_i$ links to $v_j$ 
\textit{conditional} on exposure to $v_j$.
The binary random variable $o'_{ij}$ represents
if $v_i$ links to $v_j$ assuming exposure to $v_j$;
the binary random variable $a_{ij}$ 
represents if $v_i$ is exposed to $v_j$;
and the binary random variable $o_{ij}$ 
representing if $v_i$ links to $v_j$.
Thus the data generating process for $\mathcal{G}(V, E)$ is as follows:
$\forall (i,j) \in \Da$, we have
\begin{align*}
    o'_{ij} \sim \text{Ber}(y_{ij}), \\
    a_{ij} \sim \text{Ber}(\pi_{ij}), \\
    o_{ij} = o'_{ij} a_{ij},
\end{align*}
where $\text{Ber}(.)$ is the Bernoulli distribution.
The predicted link probability is $\widehat{y}_{ij}$
and the estimated propensity is $\widehat{\pi}_{ij}$. 
The predicted link outcome is 
$\widehat{o}_{ij} = \mathbb{1}(\widehat{y}_{ij} \geq 0.5)$.
As an example, consider a citation graph. 
Here, each $v_i$ is an academic paper, 
$\pi_{ij}$ is the probability that 
authors of $v_i$ are exposed to $v_j$, 
and $y_{ij}$ is the probability 
that $v_i$ cites $v_j$ conditional on exposure to $v_j$.
\begin{definition}[\textbf{True Risk}]
This is the risk of the predictions $\widehat{y}$
on the graph that would have been generated 
if all nodes were exposed to all other nodes, i.e., 
if $\,\, \forall \,\, (i,j) \in \Da, \,\, \pi_{ij} = 1$.
The true risk is defined as
\begin{align*}
    R(\widehat{y}) &= \E_{o'}\left[ \frac{1}{|\Da|} \sum_{(i,j) \in \Da} \delta(o'_{ij}, \widehat{y}_{ij}) \right] \\
        &= \frac{1}{|\Da|} \sum_{(i,j) \in \Da} [y_{ij} \delta(1, \widehat{y}_{ij}) + (1 - y_{ij}) \delta(0, \widehat{y}_{ij}) ],
\end{align*}
where $\delta$ is some loss function (for example, log-loss).
\end{definition}
True risk is different from the risk of the predictions 
on the observed graph
as some relevant links are missing 
due to a lack of exposure.
Thus the performance of an RS should be evaluated
based on the true risk since it correctly accounts 
for relevant but low-exposure nodes.

In order to compare the biases and variances 
of the estimators we propose, 
we make Assumption~\ref{assumption:loss-fn} in this section. 
All proofs for this section are in Appendix~\ref{appendix:bias-variance}.
\begin{assumption}\label{assumption:loss-fn}
The loss function $\delta$ satisfies the following:
\begin{compactenum}
\item It only depends on the predicted binary outcome, i.e., $\delta(o_{ij}, \widehat{y}_{ij}) = \delta(o_{ij}, \widehat{o}_{ij})$,
\item $\delta(0, 0) = \delta(1, 1) = 0$, and
\item $\delta(0, 1) = \delta(1, 0) := \Delta$.
\end{compactenum}
\end{assumption}

\paragraph{Naive Estimator.}
One approach to estimating the true risk 
is to directly use the observed graph. 
We call this the naive estimator. It is defined as
\begin{align*}
    \widehat{R}_{\text{naive}}(\widehat{y}) = \frac{1}{|\Da|} \sum_{(i,j) \in \Da} \delta(o_{ij}, \widehat{o}_{ij}).
\end{align*}
\begin{lemma}\label{lemma:bias-var-naive}
The bias and variance of $\widehat{R}_{\text{naive}}(\widehat{o})$ are
\begin{align*}
    \Bias(\widehat{R}_{\text{naive}}) &= \left|\E[\widehat{R}_{\text{naive}}] - R(\widehat{o})\right| \\
        &= \frac{\Delta}{|\Da|} \left| \sum_{(i,j) \in \Da} y_{ij} (1 - \pi_{ij}) (1 - 2 \widehat{o}_{ij}) \right|,\\
    \Var(\widehat{R}_{\text{naive}}) &= \frac{\Delta^2}{|\Da|^2} \sum_{(i,j) \in \Da} y_{ij}\pi_{ij} (1 - y_{ij}\pi_{ij}).
\end{align*}
\end{lemma}
\lemref{lemma:bias-var-naive} shows that $\widehat{R}_{\text{naive}}$ is a biased estimator of the true risk. 
$\widehat{R}_{\text{naive}}$ will be unbiased 
only if either all nodes are exposed to all the others, i.e., if
$\,\, \forall (i,j) \in \Da, \,\, \pi_{ij} = 1$, 
or if all nodes are irrelevant to all the others,
i.e., if $\,\, \forall (i,j) \in \Da, \,\, y_{ij} = 0$. Thus evaluating an RS using $\widehat{R}_{\text{naive}}$ can be misleading.
We propose three estimators that leverage 
learned propensities $\widehat{\pi}$
and link probabilities $\widehat{y}$ 
to weight the examples to correct for this bias.

\paragraph{Estimator $\widehat{R}_w$.} 
The first estimator we propose is
\begin{align}
    & \widehat{R}_{w}(\widehat{y}, \widehat{\pi}) = \frac{1}{|\Da|} \sum_{(i,j) \in \Da} w_{ij} \delta(o_{ij}, \widehat{o}_{ij}), \,\, \text{where} \label{eq:risk-w-definition}\\
    & w_{ij} = \frac{o_{ij}}{\widehat{\pi}_{ij}} + (1 - o_{ij})\psi_{ij},\,\, \psi_{ij} = \frac{1-\widehat{y}_{ij}}{1-\widehat{\pi}_{ij}\widehat{y}_{ij}}. \label{eq:risk-w-weight-definition}
\end{align}
In $\widehat{R}_w$, the positive examples 
are up-weighted according to the inverse propensity. 
The negative examples are down-weighted (as $\psi_{ij} \leq 1$).
Intuitively, this weighting corrects for the fact
that, in the observed graph, 
some positive examples are observed 
as negative examples since the nodes 
are exposed according to the propensities $\pi$.
\begin{lemma}\label{lemma:rw_bias_var}
The bias and variance of $\widehat{R}_w$ are
\begin{align}
    &  \begin{aligned}[b]
        \Bias(\widehat{R}_w) = \frac{\Delta}{|\Da|} & \Bigg| \sum_{(i,j) \in \Da} \Bigg[ (1 - \widehat{o}_{ij}) y_{ij}\left(1 - \frac{\pi_{ij}}{\widehat{\pi}_{ij}} \right) + \\
        &\widehat{o}_{ij} \left( 1 - y_{ij} - (1 - y_{ij}\pi_{ij}) \psi_{ij} \right)
        \Bigg] \Bigg|,
    \end{aligned} \label{eq:risk-w-bias} \\
    & \Var(\widehat{R}_w) = \frac{\Delta^2}{|\Da|^2} \sum_{(i,j) \in \Da} y_{ij}\pi_{ij} (1 - y_{ij}\pi_{ij}) v_{ij}, \nonumber \\
    & \text{where}\,\, v_{ij} = \frac{1 - \widehat{o}_{ij}}{\widehat{\pi}^2_{ij}} + \widehat{o}_{ij} \psi_{ij}^2. \nonumber
\end{align}
\end{lemma}
\lemref{lemma:rw_bias_var} shows that $\widehat{R}_w$ will be unbiased 
if the propensities and link probabilities 
are estimated correctly, i.e., 
if $\,\, \forall (i,j) \in \Da, \,\, \widehat{\pi}_{ij} = \pi_{ij}$ 
and $\widehat{y}_{ij} = y_{ij}$. 
We later derive sufficient conditions for when $\widehat{R}_w$ 
will have lower bias than $\widehat{R}_{\text{naive}}$ 
even if $\pi$ and $y$ are incorrectly estimated.

\paragraph{Estimator $\widehat{R}_{\text{PU}}.$} 
We adapt an unbiased estimator
proposed by \citet{bekker2019beyond} 
for the positive-and-unlabeled (PU) setting.
The idea is to remove an appropriate number 
of negative examples for each positive example.
We have
\begin{align*}
    & \widehat{R}_{\text{PU}}(\widehat{y}, \widehat{\pi}) = \frac{1}{|\Da|} \sum_{(i,j) \in \Da} \left[ w_{ij} \delta(o_{ij}, \widehat{o}_{ij}) + w'_{ij} \delta(0, \widehat{o}_{ij}) \right], \\
    & \text{where}\,\, w_{ij} = \frac{o_{ij}}{\widehat{\pi}_{ij}} + (1 - o_{ij}),\, w'_{ij} = o_{ij} \left( 1 - \frac{1}{\widehat{\pi}_{ij}} \right).
\end{align*}
We weight the positive examples by the inverse propensity
and for each positive example, 
remove a negative example weighted by $|w'_{ij}|$.
\begin{lemma}\label{lemma:bias-var-pu}
The bias and variance of $\widehat{R}_{\text{PU}}$ are
\begin{align*}
    \Bias(\widehat{R}_{\text{PU}}) &= \begin{aligned}[t]
        & \frac{\Delta}{|\Da|} \left| \sum_{(i,j) \in \Da} y_{ij}\left(1 - \frac{\pi_{ij}}{\widehat{\pi}_{ij}} \right) (1 - 2 \widehat{o}_{ij}) \right|,
        \end{aligned} \\
    \Var(\widehat{R}_{\text{PU}}) &= \frac{\Delta^2}{|\Da|^2} \sum_{(i,j) \in \Da} \frac{y_{ij}\pi_{ij} (1 - y_{ij}\pi_{ij})}{\widehat{\pi}^2_{ij}}.
\end{align*}
\end{lemma}
$\widehat{R}_{\text{PU}}$ will be unbiased 
when $\,\, \forall (i,j) \in \Da, \,\, \widehat{\pi}_{ij} = \pi_{ij}$.

\paragraph{Estimator $\widehat{R}_{\text{AP}}$.} 
$\widehat{R}_{\text{AP}}$ adds positive examples for each negative example. It is defined as
\begin{align*}
    & \widehat{R}_{\text{AP}}(\widehat{y}, \widehat{\pi}) = \frac{1}{|\Da|} \sum_{(i,j) \in \Da} \left[ w_{ij} \delta(o_{ij}, \widehat{o}_{ij}) + w'_{ij} \delta(1, \widehat{o}_{ij}) \right], \\
    & \begin{aligned}
    \text{where}\,\,  & w_{ij} = o_{ij} + (1 - o_{ij}) \psi_{ij}, w'_{ij} = (1 - o_{ij}) \tau_{ij}, \\
    & \tau_{ij} = \left( \frac{\widehat{y}_{ij}(1 - \widehat{\pi}_{ij})}{1 - \widehat{\pi}_{ij} \widehat{y}_{ij}} \right).
    \end{aligned}
\end{align*}
\begin{lemma}\label{lemma:bias-var-ap}
The bias and variance of $\widehat{R}_{\text{AP}}$ are
\begin{align*}
    & \Bias(\widehat{R}_{\text{AP}}) =
        \frac{\Delta}{|\Da|} \Bigg| \sum_{(i,j) \in \Da} (1 - \widehat{o}_{ij}) [ (1-\pi_{ij}) y_{ij} - \\
        & (1-\pi_{ij}y_{ij}) \tau_{ij} ] + \widehat{o}_{ij} \left( 1 - y_{ij} - (1 - y_{ij}\pi_{ij}) \psi_{ij} \right) \Bigg|,
    \\
    & \Var(\widehat{R}_{\text{AP}}) = \frac{\Delta^2}{|\Da|^2} \sum_{(i,j) \in \Da} y_{ij}\pi_{ij} (1 - y_{ij}\pi_{ij}) \psi^2_{ij},
\end{align*}
where $\psi$ is defined in Eq. \ref{eq:risk-w-weight-definition}.
\end{lemma}
$\widehat{R}_{\text{AP}}$ is unbiased if $\,\, \forall (i,j) \in \Da, \,\, \widehat{\pi}_{ij} = \pi_{ij}$ 
and $\widehat{y}_{ij} = y_{ij}$.

\begin{theorem}[\textbf{Comparison of Variances}]\label{thm:comp_var}
For all values of $\widehat{\pi}, \widehat{y}$, we have $\Var(\widehat{R}_{\text{AP}}) < \Var(\widehat{R}_{\text{naive}}), \,\, \text{and} \\
\Var(\widehat{R}_{\text{AP}}) < \Var(\widehat{R}_w) < \Var(\widehat{R}_{\text{PU}})$.
\end{theorem}
In order to compare the biases, we make the following simplifying assumption.
\begin{assumption}\label{assumption:bias-simplification}
For the graph $\mathcal{G}(V, E)$ with $n$ nodes, 
the number of edges from each node is $\mathcal{O}(1)$.
Thus the number of positive links $|E| \in \mathcal{O}(n)$. 
And the number of negative links $(|\Da| - |E|) \in \mathcal{O}(n^2)$. 
Thus the number of negative links is much greater 
than the number of positive links for a large $n$.
If the predictions $\widehat{y}$ are close to the true values, 
we would expect the number of negative predictions 
($\widehat{o} = 0$) to also be much larger 
than the number of positive predictions ($\widehat{o} = 1$). 
So we assume that the contribution of positive predictions to the bias is negligible.
\end{assumption}
Let $\Da' = \Da \setminus E$. Under Assumption \ref{assumption:bias-simplification}, the biases are
\begin{align*}
    & \Bias(\widehat{R}_{\text{naive}}) \approx \frac{\Delta}{|\Da'|} \left| \sum_{(i,j) \in \Da'} y_{ij} (1 - \pi_{ij}) \right|, \\
    & \Bias(\widehat{R}_w) \approx \Bias(\widehat{R}_{\text{PU}}) \approx \begin{aligned}[t]
        & \frac{\Delta}{|\Da'|} \left| \sum_{(i,j) \in \Da'} y_{ij}\left(1 - \frac{\pi_{ij}}{\widehat{\pi}_{ij}} \right) \right|,
        \end{aligned} \\
    & \Bias(\widehat{R}_{\text{AP}}) \approx \begin{aligned}[t]
        \frac{\Delta}{|\Da'|} \Bigg| \sum_{(i,j) \in \Da'} \left[ (1-\pi_{ij}) y_{ij} - (1-\pi_{ij}y_{ij}) \tau_{ij} \right] \Bigg|.
    \end{aligned}
\end{align*}
\begin{theorem}[\textbf{Comparison of Biases}] \label{thm:bias-comparisons}
Under these approximations, a sufficient condition 
for $\Bias(\widehat{R}_w) = \Bias(\widehat{R}_{\text{PU}}) < \Bias(\widehat{R}_{\text{naive}})$ is
\begin{align*}
    \frac{\pi_{ij}}{2 - \pi_{ij}} < \widehat{\pi}_{ij} < 1, \,\, \forall (i, j) \in \Da,
\end{align*}
and for $\Bias(\widehat{R}_{\text{AP}}) < \Bias(\widehat{R}_{\text{naive}})$ is
\begin{align*}
    & \frac{\pi_{ij}}{2 - \pi_{ij}} < \widehat{\pi}_{ij} < 1 \,\, \text{and} \,\, 0 < \widehat{y}_{ij} < c y_{ij}, \,\, \forall (i,j) \in \Da \\
    & \text{where}\,\, c = \frac{2 (1 - \pi_{ij})}{1 - \widehat{\pi}_{ij} - \pi_{ij} y_{ij} + (2 - \pi_{ij}) \widehat{\pi}_{ij} y_{ij}} \geq 1.
\end{align*}
\end{theorem}
Thus, if $\widehat{\pi}$ are not too-underestimated 
and $\widehat{y}$ are not too-overestimated, 
the proposed estimators will have lower bias than the naive estimator.

\section{Learning Propensities and Link Probabilities}
\label{sec:learning-estimators}
The previous section assumes 
\emph{known} propensities ($\widehat{\pi}$) 
and link probabilities ($\widehat{y}$).
We present a loss function that uses our proposed estimators 
from Section~\ref{sec:estimating-risk} 
to \emph{learn} $\widehat{\pi}$ and $\widehat{y}$.
A natural approach might be to minimize 
the negative log-likelihood of the observed data:
\begin{align*}
    \widehat{\pi}, \widehat{y} = \argmin_{\widehat{\pi}, \widehat{y}} \mathcal{L}(o | \widehat{y}, \widehat{\pi}),
\end{align*}
where $\mathcal{L}(o | \widehat{y}, \widehat{\pi}) = \sum_{(i,j) \in \Da} -o_{ij} \log(\widehat{y}_{ij} \widehat{\pi}_{ij}) - (1 - o_{ij}) \log(1 - \widehat{y}_{ij} \widehat{\pi}_{ij})$. 
However, this might not ensure 
that the true risk remains small.
We derive a generalization bound 
that motivates a different loss function 
(see Appendix \ref{appendix:generalization-bound} for the proof). 

\begin{definition}[\textbf{Rademacher Complexity}]
Let $\mathcal{F}$ be a class of functions 
$(\widehat{\pi}, \widehat{y})$. 
Each estimator $\widehat{R} \in \left\{ \widehat{R}_w,  \widehat{R}_{\text{PU}}, \widehat{R}_{\text{AP}} \right\}$ can be written as $\frac{1}{|\Da|} \sum_{(i,j) \in \Da} r(o_{ij}, \widehat{\pi}_{ij}, \widehat{y}_{ij})$ 
for an appropriate function $r$
(e.g. by Eq.~\ref{eq:risk-w-weight-definition}, 
for $\widehat{R}_w$, we have 
$r(o_{ij}, \widehat{\pi}_{ij}, \widehat{y}_{ij}) = w_{ij} \delta(o_{ij}, \widehat{y}_{ij})$)
. 
For $\widehat{R} \in \left\{ \widehat{R}_w,  \widehat{R}_{\text{PU}}, \widehat{R}_{\text{AP}} \right\}$, we define a quantity analogous to the
Empirical Rademacher Complexity \citep{bartlett2002rademacher} as
\begin{align*}
    \widehat{\mathcal{G}}_o(\mathcal{F}, \widehat{R}) &= \E_{\sigma} \left[ \sup_{(\widehat{\pi}, \widehat{y}) \in \mathcal{F}} \frac{1}{|\Da|} \sum_{(i,j) \in \Da} \sigma_{ij} r(o_{ij}, \widehat{\pi}_{ij}, \widehat{y}_{ij}) \right],
\end{align*}
where $\sigma_{ij}$ are independent Rademacher random variables. 
And the Rademacher Complexity is 
$\mathcal{G}(\mathcal{F}, \widehat{R}_w) = \E_{o}[\widehat{\mathcal{G}}_o(\mathcal{F}, \widehat{R}_w)]$.
\end{definition}
$\widehat{\mathcal{G}}_o(\mathcal{F}, \widehat{R}_w)$ can be 
estimated from the data
by taking a random sample of the variables $\sigma_{ij}$
and optimizing the above objective. 
Next, we present a generalization bound 
based on $\widehat{\mathcal{G}}_o(\mathcal{F}, \widehat{R}_w)$. 
\begin{theorem}[\textbf{Generalization Bound}] \label{thm:generalization-bound}
Let $\mathcal{F}$ be a class of functions 
$(\widehat{\pi}, \widehat{y})$.
Let $\delta(o_{ij}, \widehat{y}_{ij}) \leq \eta\,\,\forall (i,j) \in \Da$ 
and $\widehat{\pi}_{ij} \geq \epsilon > 0\,\,\forall (i,j) \in \Da$.
Then, for $\widehat{R} \in \left\{ \widehat{R}_w,  \widehat{R}_{\text{PU}}, \widehat{R}_{\text{AP}} \right\}$, with probability at least $1 - \delta$, we have
\begin{align*}
    R(\widehat{y}) &\leq \widehat{R}(\widehat{y}, \widehat{\pi}) + B(\widehat{R}) + 2 \mathcal{G}(\mathcal{F}, \widehat{R}) + M \\
        &\leq \widehat{R}(\widehat{y}, \widehat{\pi}) + B(\widehat{R}_w) + 2 \widehat{\mathcal{G}}(\mathcal{F}, \widehat{R}_w) + 3 M, 
\end{align*}
where $M = \sqrt{\frac{4 \eta^2}{\epsilon^2 |\Da|} \log(\frac{2}{\delta})}$ and $B(\widehat{R})$ is the bias of $\widehat{R}$ derived in Section~\ref{sec:estimating-risk}.
\end{theorem}
\textbf{Loss Function.}
The bound shows that $\widehat{R} \in \{\widehat{R}_w, \widehat{R}_{\text{PU}}, \widehat{R}_{\text{AP}}\}$ is close to the true risk $R$.
This suggests that we should choose $\widehat{\pi}, \widehat{y}$ 
that lead to small values of $\widehat{R}$
as this will also minimize the true risk with high probability.
This motivates us to learn $\widehat{\pi}, \widehat{y}$ 
by minimizing the following objective:
\begin{align*}
    \widehat{\pi}, \widehat{y} = \argmin_{\widehat{\pi}, \widehat{y}} \mathcal{L}(o | \widehat{y}, \widehat{\pi}),\,\,\text{subject to}\,\, \widehat{R}(\widehat{\pi}, \widehat{y}) \leq c,
\end{align*}
where $\widehat{R} \in \{\widehat{R}_w, \widehat{R}_{\text{PU}}, \widehat{R}_{\text{AP}}\}$
and $c > 0$ is some constant. 
In practice, we minimize the following 
relaxed version of this objective:
\begin{align}\label{eq:combined-loss-function}
    l(\widehat{\pi}, \widehat{y}) = \lambda_L \mathcal{L}(o|\widehat{\pi}, \widehat{y}) + \lambda_R \widehat{R}(\widehat{\pi}, \widehat{y}),
\end{align}
where $\lambda_R$ and $\lambda_L$ are hyperparameters. 
One might try to minimize the loss function using only
$\widehat{R}$
by setting $\lambda_L = 0$.
\textcolor{black}{This will not work because trivial solutions exist for all three risk functions:
if $\,\,\forall (i,j) \in \Da, \,\, \widehat{y}_{ij} = 1$, then $\widehat{R}_w(\widehat{y}, \widehat{\pi}) = 0$;
if $\,\,\forall (i, j) \in \Da, \,\, \widehat{\pi}_{ij} = 1, \widehat{y}_{ij} > 0.5$, then $\widehat{R}_{\text{PU}} = \widehat{R}_{\text{AP}} = 0$.
Hence, we need to use $\lambda_L > 0$ during training to prevent the model from collapsing to these solutions.}
It is possible use parametric models 
like neural networks for $\widehat{y}$ and $\widehat{\pi}$
to incorporate information associated with the nodes 
(like user data or paper data). 
The parameters can be learned 
by using gradient-based methods 
by minimizing the loss in Eq.~\ref{eq:combined-loss-function}.

\section{Feedback Loops}
\label{sec:feedback-loops}
In this section, we analyze what happens 
when we train an RS 
repeatedly on data generated by users 
interacting with that system's recommendations. 
We show that for an RS 
that does not account for exposure bias,
the fraction of high-propensity nodes
that are recommended
continually increases over time. 
In other words, the system will
progressively recommend 
fewer low-propensity nodes,
even if they are relevant, as time goes on.
Next, we show that correcting for exposure bias ensures 
that relevant low-propensity nodes keep being recommended. 
In this section, we assume that the attributes 
of the nodes take values in a discrete set.
\begin{assumption}\label{assumption:feedback-loop-notation}
Each node belongs to one of $C$ categories 
from the set $\mathcal{C} = \{c_1, \hdots, c_C \}$.
Each category contains $n$ nodes. 
$V = \{ v_1, \hdots, v_N \}$ is the set of nodes and $N = n C$.
The function $\gamma: V \to \mathcal{C}$
maps a node to its category. 
The link probability $y_{ij}$ and propensity $\pi_{ij}$ 
depend only on the categories of the nodes, i.e.,
$y_{ij} = y_{lm}$ and $\pi_{ij} = \pi_{lm}$ if
$\gamma(v_i) = \gamma(v_l)$ and $\gamma(v_j) = \gamma(v_m)$.
Therefore, for any pair of nodes $(v_i, v_j)$,
the product $\pi_{ij} y_{ij}$ depends only
on the categories $v_i$ and $v_j$ belong to. 
Let $q_{uv} = \pi_{ij} y_{ij}$ for some $v_i, v_j$ 
s.t. $\gamma(v_i) = c_u$ and $\gamma(v_j) = c_v$.
\end{assumption}

\paragraph{Iterative Training Process.}
We now describe the iterative training process 
for an RS that does not account for exposure bias. 
We will restrict our attention to analyzing 
the recommendations made for the $n$ nodes 
in some category $c_u \in \mathcal{C}$.
We assume that we make one recommendation 
for each node 
(this simplifies exposition but is not necessary).
At time step $t$, the fraction of nodes recommended 
from each category is represented 
by the $(C-1)$-simplex $\kappa^{(t)}$. 
So out of the $n$ nodes from $c_u$,
$n \kappa^{(t)}_v$ of them are recommended 
a node from the category $c_v$, 
where $\kappa^{(t)}_v$ is the $v^{\text{th}}$ element of $\kappa^{(t)}$.
Links are generated from the recommended nodes 
according to ground-truth propensities and link probabilities.
Thus a node from $c_u$ creates a link 
to a recommended node from $c_v$
with probability $q_{uv}$.
Since we are examining recommendations for category $c_u$,
we will drop the subscript $u$ going forward, 
i.e., $q_v = q_{uv}$.
This gives us training data for the next iteration.
We assume that a node only creates a link 
to nodes from the recommended nodes. 
In other words, links are not created 
to nodes that are not recommended.
The number of nodes linked to 
from category $c_v$ at time $t$ is $n^{(t)}_v$.
Then $n^{(t)}_v \sim \text{Binomial}(n \kappa^{(t)}_v, q_v)$.
During training, the link probability is estimated 
as $\widehat{q}^{(t)}_v = \frac{n^{(t)}_v}{n}$.
We assume that, at time step $t+1$,
nodes from category $c_v$
are recommended with probability 
proportional to $\widehat{q}^{(t)}_v$.
This is akin to recommending
with some exploration \citep{kawale2015efficient}.
Let the $(C-1)$-simplex denoting normalized estimates be $\widehat{e}^{(t+1)} = \left[\frac{\widehat{q}^{(t)}_1}{S}, \frac{\widehat{q}^{(t)}_2}{S}, \hdots, \frac{\widehat{q}^{(t)}_C}{S} \right]$,
where $S = \sum_{j=1}^{C}\widehat{q}^{(t)}_j$.
Thus the recommendations for the next step $\kappa^{(t+1)}$ 
have the distribution $\kappa^{(t+1)} \sim \frac{1}{n} \textit{Multinomial}(n, \widehat{e}^{(t+1)})$.
This process is repeated at each time step. 
The initial training data is generated
by the user generating links according 
to the ground-truth propensities and link probabilities.

\begin{example}
We illustrate the iterative training process with a minimal example. 
Let $\mathcal{C} = \{c_1, c_2\}$.
We examine the recommendations made to nodes in $c_1$. 
Let $n = 100, q_{1,1} = 0.8$ and $q_{1,2} = 0.4$.
At time $t$, let $\kappa^{(t)} = [0.6, 0.4]$. 
Informally, $60$ of the recommended nodes
are from $c_1$ and the remaining $40$ from $c_2$.
The nodes create links to the recommended nodes
with probabilities $q_{1,1}$ and $q_{1,2}$. 
Therefore, the number of nodes linked 
that belong to $c_1$ at time $t$ 
is $n^{(t)}_1 \sim \text{Binomial}(60, q_1=0.8)$ 
and similarly, $n^{(t)}_2 \sim \text{Binomial}(40, q_2=0.4)$. 
Informally, the realized values 
are $n^{(t)}_1 = 48$ and $n^{(t)}_2 = 16$.
The estimated link probabilities 
are $\widehat{q}^{(t)}_1 = 0.48, \widehat{q}^{(t)}_2 = \frac{16}{100} = 0.16$ and
$\widehat{e}^{(t+1)} = \left[ \frac{0.48}{0.64}, \frac{0.16}{0.64} \right] = [0.75, 0.25]$.
Then, at time $t+1$, we recommend nodes
according to $\widehat{e}^{(t+1)}$, i.e., 
$\kappa^{(t+1)} \sim \frac{1}{100} \text{Multinomial}(100, \widehat{e}^{(t+1)})$. 
The realized value of $\kappa^{(t+1)}_1$ is likely 
to be greater than $\kappa^{(t)}_1$.
Thus more items from $c_1$ are likely 
to be recommended at time $t+1$ as compared to time $t$. 
This provides some intuition 
for the existence of a feedback loop:
nodes that are linked less are in turn recommended
with a lower probability in the next time step.
\end{example}

We formally show the existence of feedback loops 
(see Appendix \ref{appendix:feedback-loops} for the proofs).
We prove a finite-sample result which shows that,
with high probability, the relative probability 
of recommending nodes from categories
with higher values of $q_j$ 
keeps increasing over time.

\begin{theorem}\label{thm:feedback-loop-finite-sample}
Suppose that $q_v > q_w$ if $v > w$. 
Let $\kappa^{(t)}_{vw} = \frac{\kappa^{(t)}_v}{\kappa^{(t)}_v + \kappa^{(t)}_w}$.
Let $A^{(t)}_{vw}$ represent the event 
that relative fraction of recommendations 
from $c_v$ to that from $c_w$ increases at time $t$,
i.e., $\kappa^{(t+1)}_{vw} > \kappa^{(t)}_{vw}$. 
Let $A^{(t)}$ be the event
that all relative fractions get skewed
towards $c_v$ from $c_w$ if $q_v > q_w$, 
i.e., $A^{(t)} = \bigcap_{(v,w) \in \mathcal{S}} A^{(t)}_{vw}$,
where $\mathcal{S} = \{ (v, w): v \in [C], w \in [C], v > w \}$.
Then, for constants $\epsilon, \eta > 0$ 
that only depend on $\kappa^{(t)}$ and $q$, we have
\begin{align*}
    & \P(A^{(t)} |\kappa^{(t)}) \begin{aligned}[t]
        & \geq 1 - 2C \exp\left( -2 n \left[ \epsilon^2 + \frac{\eta^2}{C^2} \right]  \right) \\
        & \geq 1 - 2C \exp\left( - \mathcal{O}\left(\frac{n}{C^2}\right) \right).
    \end{aligned}
\end{align*}
\end{theorem}
\begin{cor}
$\lim\limits_{n \to \infty} \P(A^{(t)}|\kappa^{(t)}) = 1$ if $C^3 \in o(n)$.
\end{cor}
Thus, at each time step, with high probability,
nodes with low propensity are less likely 
to be recommended in the next time step. 
Therefore, if an RS does not correct 
for exposure bias, over time,
even relevant nodes with low propensity 
are unlikely to be recommended. 
Next, we derive and analyze the rate 
at which the exposure bias exacerbates.
\begin{theorem}\label{thm:worsen_rate}
Suppose that $q_v > q_w$. As $n \rightarrow \infty$, $\kappa^{(t)}_{vw} \overset{p}{\rightarrow} 1 - \frac{1}{1 + c^t}$, where $c = \frac{q_v}{q_w}$.
\end{theorem}
\thmref{thm:worsen_rate} shows that
the rate at which the bias exacerbates
is dependent on the ratio $\frac{q_v}{q_w}$.
Therefore, the lower the propensity, 
the faster the probability of that node
being recommended reduces.

\begin{cor}
Let $y_{c_u c_v} = y_{ij}$ for some $(i, j)$ 
s.t. $u = \gamma(i)$ and $v = \gamma(j)$
($\gamma$ is defined after Assumption \ref{assumption:feedback-loop-notation}).
We now assume that we have a consistent estimator
$\widehat{q}^{(t)}_v \overset{p}{\rightarrow} \kappa^{(t)}_v y_{g_u g_v}$, 
where $\kappa^{(t)}_v$ is the $v^{\text{th}}$ element of the simplex $\kappa^{(t)}$. 
Thus $\widehat{q}^{(t)}_v$ is an estimator 
that negates the effect of exposure bias.
As $n \rightarrow \infty$, $\kappa^{(t)}_{vw} \overset{p}{\rightarrow} 1 - \frac{1}{1 + c^{t}}$, where $c = \frac{y_{g_u g_v}}{y_{g_u g_w}}$.
\end{cor}
This shows that accounting for exposure bias 
can alleviate the feedback loop. 
Despite having low propensity, 
relevant papers will continue 
to be recommended.

\section{Experiments}
\label{sec:experiments}
We validate our link recommendation methods
on the task of citation recommendation. 
Given an input paper's data 
(like title, abstract, etc.),
the goal is to recommend papers that it should cite.
We use the Microsoft Academic Graph (MAG) 
dataset \citep{sinha2015overview}.
MAG is a graph containing scientific papers 
and the citation relationships between them. 
It also contains the titles, 
abstracts, and FOS of the papers.
In our experiments, we use subgraphs from the MAG 
by performing a breadth-first search from some root node.
For each paper, we concatenate the title and abstract
and generate a 768-dimensional embedding for the text 
using the \textit{bert-as-service} library \citep{xiao2018bertservice}. 
We use a SciBERT model \citep{Beltagy2019SciBERT}, 
which is a BERT model trained on scientific text, with this library.
For each paper $p_i$, we generate 
the embedding $h_i \in \RB^{768}$.
The FOS in MAG are organised as a tree,
where a child is a sub-field of its parent.
We only use the root-level FOS for each paper
and there are 19 such FOS. 
We use Amazon Sagemaker \citep{liberty2020elastic}
to run our experiments.

For simplicity, we assume that the propensities $\pi_{ij}$ 
depend only on the FOS of papers $p_i$ and $p_j$.
Thus the propensity model is parameterized by 
$\widehat{\theta}_{\pi} \in [0,1]^{19 \times 19}$.
However, our methods can easily extend 
to more complicated parametric propensity estimators like neural networks.
To model the link probability $\widehat{y}_{ij}$,
we use the following model:
\begin{align}\label{eq:experiments-output-model}
    \widehat{y}_{ij} = \sigma(\widehat{w}^\top (h_i \odot h_j) + \widehat{b}),
\end{align}
where $\widehat{w} \in \RB^{768}$ 
and $\widehat{b} \in \RB$ are trainable parameters,
$\odot$ is an element-wise product, 
and $\sigma$ is the sigmoid function.
We use stochastic gradient descent 
to learn $\widehat{\theta}_{\pi}, \widehat{w}$ 
and $\widehat{b}$ using the loss function 
described in Eq. \ref{eq:combined-loss-function} 
with $\delta$ as the log-loss, i.e., 
$\delta(u, \widehat{u}) = -u \log(\widehat{u}) - (1-u) \log(1-\widehat{u})$.
For training, we use the Adam optimizer \citep{kingma2014adam} 
with a learning rate of $10^{-4}$ and a batch size of $32$.

\subsection{Semi-Synthetic Dataset} \label{sec:experiments-semi-synthetic}

\begin{table}
\caption{Evaluation metrics 
on the test set of the semi-synthetic data computed against
known ground truth citation links.
}
\label{table:semi-synthetic-eval-metrics}
\begin{center}
\begin{small}
\begin{sc}
\begin{tabular}{lcccccr}
\toprule
Model & Prec. & Rec. & AUC & MAP \\
\midrule
\makecell[l]{No Prop.} & 67.24 & 54.81 & 84.45 & 41.87 \\
\makecell[l]{MLE} & 81.04 & 60.19 & 93.12 & 56.77 \\
\makecell[l]{$\widehat{R}_w$} & \textbf{83.28} & 63.73 & \textbf{96.42} & 56.96 \\
\makecell[l]{$\widehat{R}_\text{PU}$} & 82.16 & 63.07 & 94.28 & 58.01 \\
\makecell[l]{$\widehat{R}_\text{AP}$} & 83.01 & \textbf{65.54} & 95.38 & \textbf{59.90} \\
\bottomrule
\end{tabular}
\end{sc}
\end{small}
\end{center}
\end{table}

Since we do not have ground truth exposure values in the MAG dataset, 
we cannot know whether a paper was not cited
due to a lack of exposure or due to irrelevancy.
As a result, we construct a semi-synthetic dataset 
with simulated propensity scores and link probabilities.
We use a subset of $41{,}600$ papers.
We generate train-test-validation splits 
by taking a topological ordering of the nodes 
and use the subgraph created 
from the first $70\%$ for training, 
next $10\%$ for validation, 
and the remaining $20\%$ for testing.
We use the real text and FOS for each paper.
The simulated propensity matrix $\pi$ 
is a $19 \times 19$ matrix 
with its diagonal and off-diagonal entries initialized
from $U(0.7, 1)$ and $U(0.1, 0.3)$, respectively, 
where $U(.)$ is the uniform distribution.
The link probability is simulated using 
$y_{ij} = \sigma(w^\top (h_i \odot h_j) + b)$,
where $\sigma$ is the sigmoid function, 
$\odot$ is element-wise product, 
and $w, b$ are fixed \emph{known} vectors.

\begin{table}
\caption{RMSE of the estimated risk 
with respect to the true risk 
computed using our proposed estimators.
The first column shows the risk 
used in the loss function in Eq.~\ref{eq:combined-loss-function} 
to learn $\widehat{\pi}$ and $\widehat{y}$.
}
\label{table:semi-synthetic-rmse-metrics}
\begin{center}
\begin{small}
\begin{sc}
\begin{tabular}{l|ccccr}
\toprule
\multirow{2}{*}{\makecell{Trained \\ Using}} & \multicolumn{4}{c}{Estimator Used} \\
& $\widehat{R}_\text{naive}$ & $\widehat{R}_w$ & $\widehat{R}_\text{PU}$ & $\widehat{R}_\text{AP}$ \\
\midrule
No Prop. & 1.50 & - & - & - \\
\makecell[l]{MLE} & 0.67 & \textbf{0.23} & 0.24 & 0.32 \\
\makecell[l]{$\widehat{R}_w$} & 0.43 & \textbf{0.04} & 0.10 & 0.11 \\
\makecell[l]{$\widehat{R}_\text{PU}$} & 0.38 & 0.05 & 0.11 & \textbf{0.04} \\
\makecell[l]{$\widehat{R}_\text{AP}$} & 0.41 & 0.06 & 0.08 & \textbf{0.03} \\
\bottomrule
\end{tabular}
\end{sc}
\end{small}
\end{center}
\end{table}

We show 
the evaluation metrics for five models on the test set 
computed against the simulated true citations 
(not the observed citations) (Table \ref{table:semi-synthetic-eval-metrics}).
\textit{No Prop} is the model trained naively on the observed data using only
the output model in Eq. \ref{eq:experiments-output-model}. 
\textit{MLE} is the model trained using the loss function 
in Eq. \ref{eq:combined-loss-function} with $\lambda_R = 0$.
The remaining three are models trained using
$\widehat{R}_w, \widehat{R}_{\text{PU}}$, 
and $\widehat{R}_{\text{AP}}$ with $\lambda_R = 10$ and $\lambda_L = 1$.
We see that all other estimators 
significantly outperform \textit{No Prop}.
Additionally, our proposed estimators 
lead to improved performance over the MLE. 
We emphasize that these metrics 
are computed against true citations 
and thus are a measure of true risk 
which is the appropriate metric 
for evaluating an RS's performance.
This shows the utility of accounting for exposure bias 
and learning using our proposed loss function.

\begin{figure}[t]
\begin{center}
\begin{subfigure}{0.47\columnwidth}
\includegraphics[width=0.9\columnwidth]{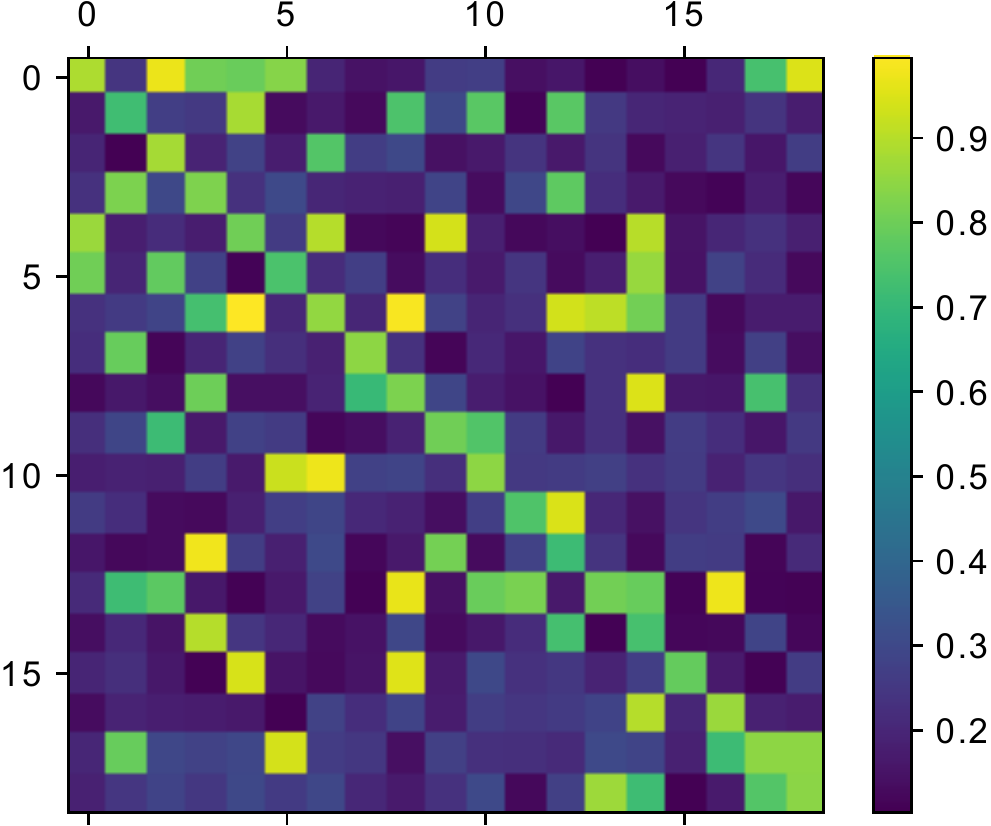}
\caption{Ground Truth}
\end{subfigure}
\begin{subfigure}{0.47\columnwidth}
\includegraphics[width=0.9\columnwidth]{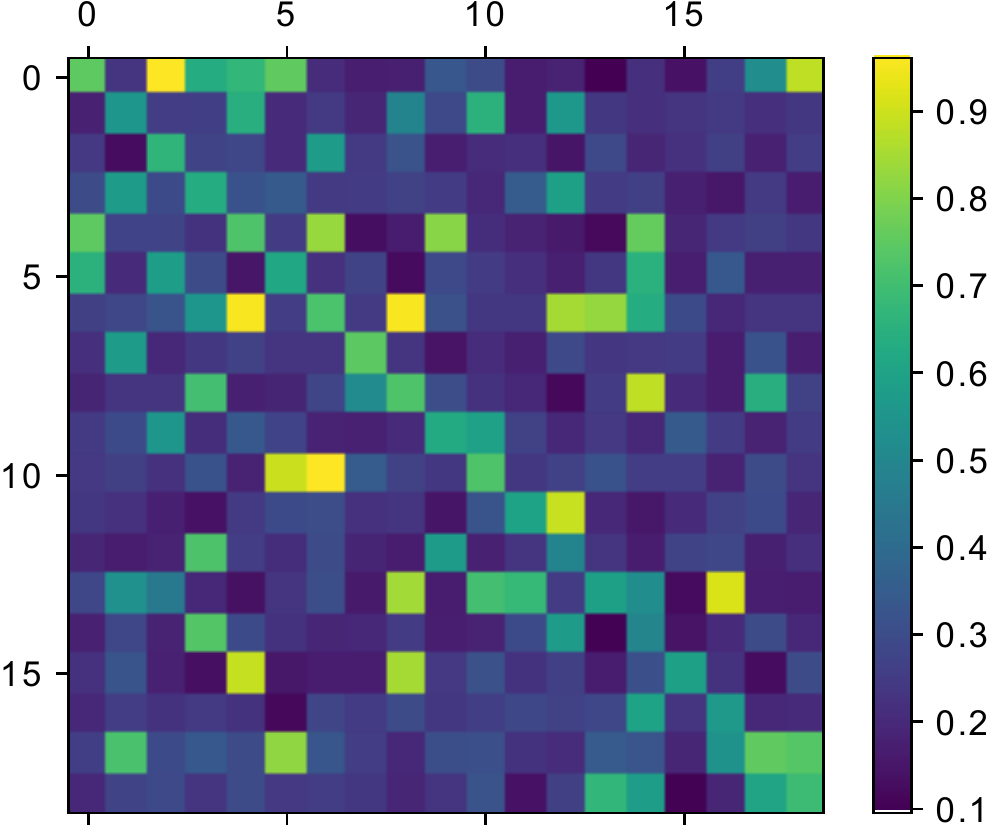}
\caption{Estimated}
\end{subfigure}
\caption{The estimated propensities propensities are close to the true simulated values when learned using $\widehat{R}_w$.}
\label{fig:ground-truth-and-estimated-propensities}
\end{center}
\end{figure}

In this work, we tackle two separate (but related) challenges.
The first challenge is \textit{learning} link probabilities in such a way that they are not underestimated due to exposure bias.
The second challenge is \textit{evaluating} a RS given learned link probabilities and propensity scores, i.e., computing a good estimate of the true risk.
We demonstrate the efficacy of our methods for the second challenge and show
that our proposed weighting schemes lead to good estimates of the true risk (Table~\ref{table:semi-synthetic-rmse-metrics}).
We show 
the RMSE of the risk estimated 
using the proposed estimators
with respect to the true risk. 
The first column denotes risk function 
used to train the model 
(as described in Section~\ref{sec:learning-estimators}).
The rest of the columns denote 
the estimators used to estimate true risk 
using the learned propensities and link probabilities 
from the trained model in the first column.
We trained each model $10$ times to compute the RMSE.
The RMSE estimated using $\widehat{R}_{\text{naive}}$ 
is always greater than that of the other estimators, 
which shows that leveraging the learned propensities 
leads to substantially better estimates of the true risk 
(and thus more accurately evaluates the RS). 
The RMSE when trained using the \textit{MLE}
is higher than when trained using our proposed estimators, 
showing the benefit of our proposed estimators over the \textit{MLE}.
This also qualitatively validates the generalization bound
proved in Section~\ref{sec:learning-estimators} 
by showing that minimizing 
$\widehat{R} \in \{ \widehat{R}_w, \widehat{R}_{\text{PU}}, \widehat{R}_{\text{AP}} \}$
also leads to small values of the true risk.

The heatmap of simulated propensities 
and estimated propensities when using $\widehat{R}_w$
shows that the estimated propensities 
are close to the true propensities (Figure \ref{fig:ground-truth-and-estimated-propensities}). 
The mean relative error between the true 
and estimated propensities is $19.47\%$, 
demonstrating that the training procedure recovers the propensities.
Together, these results show that our methods 
successfully mitigate exposure bias in this dataset.

\subsection{Real-World Datasets}\label{sec:real-world-datasets}

\begin{table}
\caption{Evaluation metrics for various models 
computed on the test sets 
of the two real-world citation datasets.}
\label{table:real-data-standard-metrics}
\begin{center}
\begin{small}
\begin{sc}
\begin{tabular}{lccccccr}
\toprule
Model & Prec. & Rec. & F1 & AUC & \makecell[l]{MAP}\\
\midrule
\multicolumn{6}{l}{\textbf{Dataset 1}} \\
\midrule
\makecell[l]{No Prop.} & 29.45 & 78.30 & 42.81 & 84.44 & 24.10 \\
\makecell[l]{MLE} & 30.24 & 77.84 & 43.56 & 84.41 & 24.60 \\
\makecell[l]{$\widehat{R}_w$} & 31.46 & 78.02 & 44.84 & 84.74 & 25.60  \\
\makecell[l]{$\widehat{R}_\text{PU}$} & 30.98 & \textbf{78.94} & 44.49 & \textbf{85.24} & 25.11 \\
\makecell[l]{$\widehat{R}_\text{AP}$} & \textbf{36.07} & 76.08 & \textbf{48.94} & 84.67 & \textbf{28.58} \\
\midrule
\multicolumn{6}{l}{\textbf{Dataset 2}} \\
\midrule
\makecell[l]{No Prop.} & 44.86 & 70.85 & 54.94 & 83.22 & 33.19 \\
\makecell[l]{MLE} & 44.43 & 74.66 & 55.71 & 84.97 & 34.39 \\
\makecell[l]{$\widehat{R}_w$} & \textbf{48.70} & 71.62 & \textbf{57.98} & 83.90 & \textbf{36.25} \\
\makecell[l]{$\widehat{R}_\text{PU}$} & 42.17 & \textbf{76.15} & 54.28 & \textbf{85.43} & 33.26 \\
\makecell[l]{$\widehat{R}_\text{AP}$} & 47.22 & 71.84 & 56.98 & 83.89 & 35.27 \\
\bottomrule
\end{tabular}
\end{sc}
\end{small}
\end{center}
\end{table}
We now evaluate our proposed method 
on a real-world citation network. 
We construct two datasets
by using disjoint subgraphs of the MAG. 
The first generated dataset has
$2{,}442{,}008$ papers and $7{,}577{,}886$ edges. 
The second dataset has $1{,}328{,}664$ papers and $1{,}469{,}899$ edges. 
Thus the second graph is sparser than the first one.
The FOS distribution is also different in both datasets 
(see details in Appendix \ref{appendix:experiments}). 
We use 70-10-20\% train-validation-test splits 
generated similarly to the semi-synthetic dataset.
We do not have access to true exposure values
and thus we evaluate our methods against the observed citation links.

We show 
the evaluation metrics for the proposed estimators (Table \ref{table:real-data-standard-metrics}).
Since we do not have access 
to the true citation links in the real dataset,
we compute these metrics over the observed links. 
In other words, this is a measure of the naive risk.
We see that our proposed estimators 
achieve comparable metrics 
to \textit{No Prop}.
For both datasets, the best numbers for each metric 
are achieved by estimators other than \textit{No Prop}.
Moreover, the models using the weighted estimators 
outperform the \textit{MLE} estimator in both datasets.
These results show that our proposed estimators 
achieve comparable performance 
even when evaluated on the observed citation data.
Similarly, Table \ref{table:real-data-link-pred-metrics} shows 
link prediction metrics
for various models computed against \textit{observed} citations.
\textit{Recall@100} refers to the recall in the top 100 recommendations 
averaged across all papers in the test set. 
\textit{Mean Rank} is the mean rank 
of the cited papers
averaged across all the papers.
\textit{Entropy@100 of True Positives} is the entropy in the FOS 
of the true positives in the top 100 recommendations for each paper; 
we use it to measure the diversity in the FOS of the recommendations.
Our proposed estimators achieve comparable \textit{Recall@100} and \textit{Mean Rank} to \textit{No Prop}
for both datasets.
As expected, propensity based estimators have higher FOS entropy scores than \textit{No Prop}, with $\widehat{R}_w$ achieving the highest FOS entropy in both datasets. Thus our proposed estimators \emph{recommend more relevant papers from different FOS and still maintain comparable performance} to \textit{No Prop}.

At first blush, the comparable performance to \textit{No Prop} may not seem compelling. However, this is a strong result.
Our goal is to correct exposure bias and minimize \textit{true risk}, not observed (or naive) risk.
Since Tables \ref{table:real-data-standard-metrics} and \ref{table:real-data-link-pred-metrics} are computed against observed citations, our proposed methods should not be expected to outperform \textit{No Prop} as they are not trying to optimize metrics against the observed links.
In Section \ref{sec:experiments-semi-synthetic}, we showed that our methods correct exposure bias and achieve lower true risk. Coupled with those results, our goal in this section was to show that our methods do not lower performance even if evaluated against the standard evaluation metrics.
\textcolor{black}{We suspect that the negative log-likelihood term $\mathcal{L}(o|\widehat{\pi}, \widehat{y})$ in Eq.~\ref{eq:combined-loss-function} is likely responsible for the comparable performance against the observed risk. This is because, as seen from the results, the MLE also performs well as compared to \textit{No Prop}.}.

\begin{table}
\caption{Link prediction metrics for various models when evaluated on the test set of the real datasets.}
\label{table:real-data-link-pred-metrics}
\begin{center}
\begin{small}
\begin{sc}
\begin{tabular}{lcccccr}
\toprule
Model & \makecell{Recall\\@100} & \makecell{Mean \\ Rank} & \makecell{Entropy@100 \\ True positives}\\
\midrule
\multicolumn{4}{l}{\textbf{Dataset 1}} \\
\midrule
\makecell[l]{No Prop.} & 24.39 & \textbf{2247.27} & 1.65 \\
\makecell[l]{MLE} & 24.70 & 2891.40 & 1.73 \\
\makecell[l]{$\widehat{R}_w$} & 25.03 & 2836.73 & \textbf{1.74} \\
\makecell[l]{$\widehat{R}_\text{PU}$} & 24.61 & 2875.13 & 1.73 \\
\makecell[l]{$\widehat{R}_\text{AP}$} & \textbf{26.66} & 2425.51 & 1.71 \\
\midrule
\multicolumn{4}{l}{\textbf{Dataset 2}} \\
\midrule
\makecell[l]{No Prop.} & \textbf{6.32} & \textbf{10170.26} & 1.06 \\
\makecell[l]{MLE} & 6.07 & 10731.88 & 1.08 \\
\makecell[l]{$\widehat{R}_w$} & 6.30 & 10873.19 & \textbf{1.12} \\
\makecell[l]{$\widehat{R}_\text{PU}$} & 5.92 & 10717.06 & 1.08 \\
\makecell[l]{$\widehat{R}_\text{AP}$} & 5.99 & 10801.11 & 1.10 \\
\bottomrule
\end{tabular}
\end{sc}
\end{small}
\end{center}
\end{table}

\subsection{Feedback Loops}

\begin{figure}[t]
\begin{center}
\begin{subfigure}{0.49\columnwidth}
\includegraphics[width=1\columnwidth]{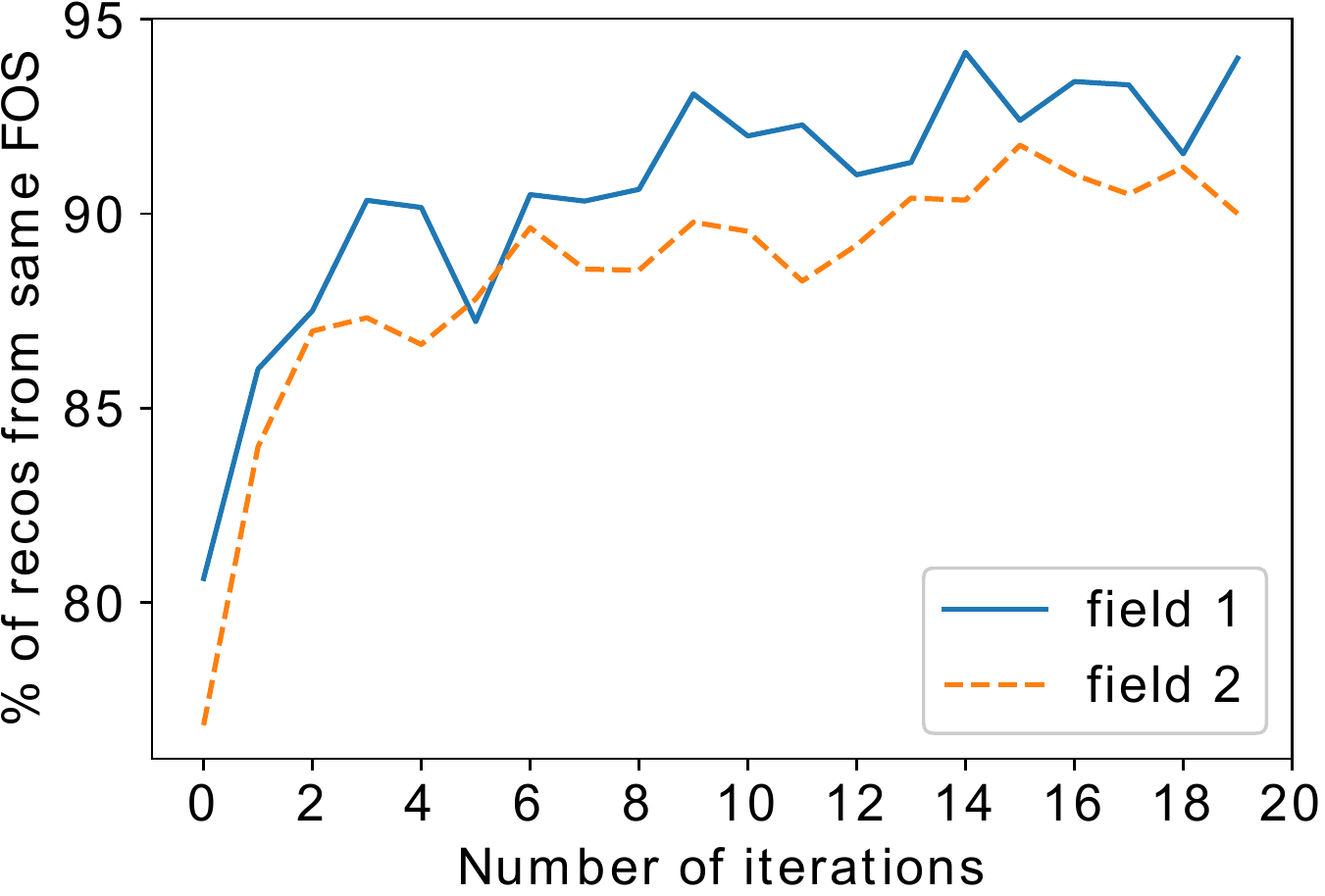}
\caption{No Propensity}
\label{fig:feedback-loops-fos-no-propensity}
\end{subfigure}
\begin{subfigure}{0.49\columnwidth}
\includegraphics[width=1\columnwidth]{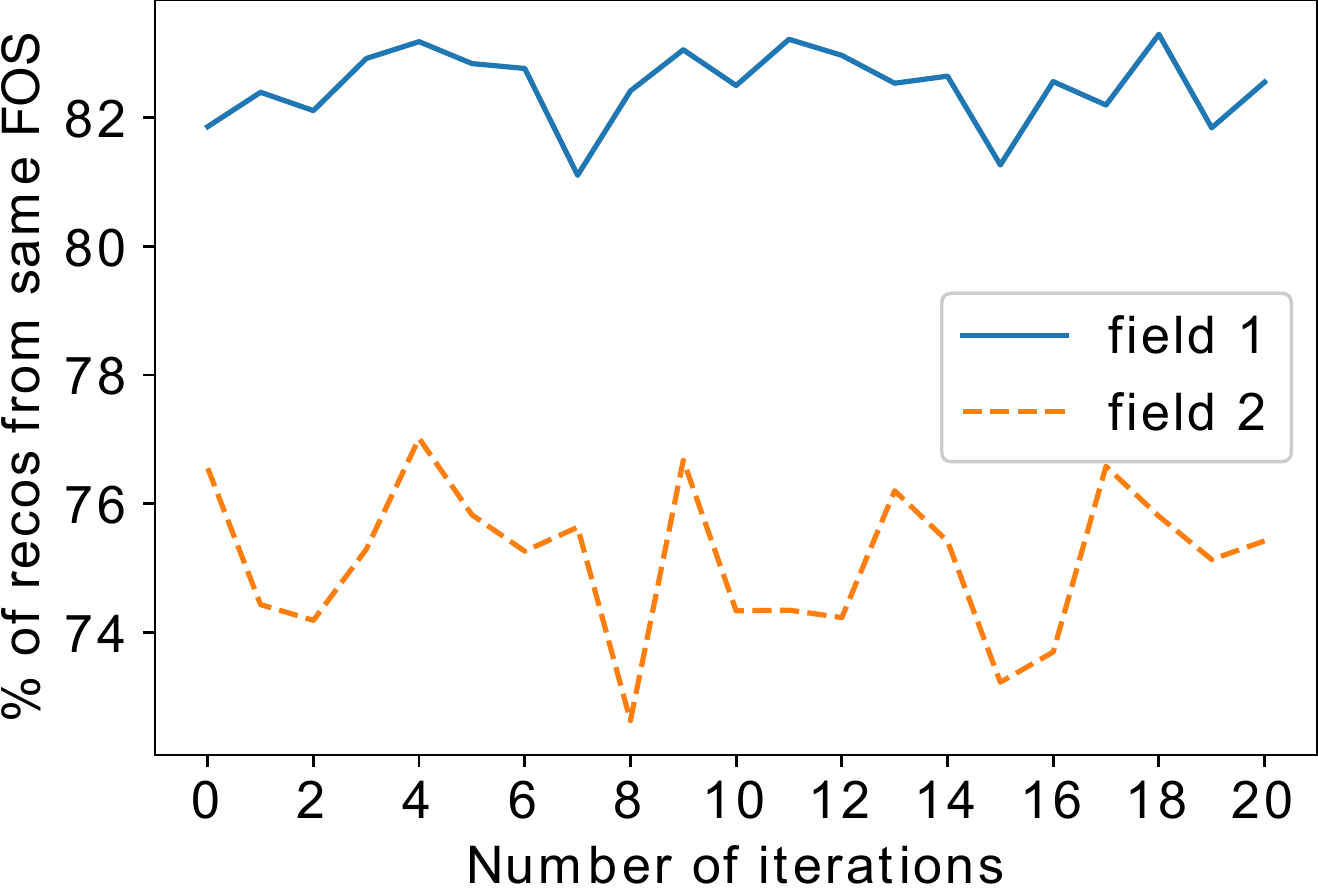}
\caption{With Propensity ($\widehat{R}_w$)}
\label{fig:feedback-loops-fos-with-propensity}
\end{subfigure}
\caption{The fraction of recommended papers from the same FOS over time.}
\label{fig:feedback-loops-fos}
\end{center}
\end{figure}

We run simulations to examine what happens 
when a citation recommender is trained repeatedly 
on data collected from users
interacting with its recommendations. 
We use the iterative training procedure 
described in Section~\ref{sec:feedback-loops}. 
We construct a training set of $410$ papers from the MAG,
with their corresponding real FOS and text embeddings. 
For ease of exposition, we use an arbitrary 
but fixed mapping to map the 19 FOS to two FOS. 
The synthetic propensities and link probabilities
are simulated similarly to Sec~\ref{sec:experiments-semi-synthetic}.
In the first iteration, the models are trained 
using the observed citation network. 
For subsequent training iterations, 
the training data is generated as follows.
For each paper, we recommend 20 papers. 
The probability of recommending a paper is proportional 
to its estimated citation probability.
We then simulate the user's interaction with the recommendations 
according to the known simulated propensities and citation probabilities.
This generates the training set for the subsequent iteration.
We then repeat this process.

We show how the fraction of recommended papers 
from the same FOS changes over multiple training iterations 
for models trained without propensity, i.e., \textit{No Prop} 
and the model trained using $\widehat{R}_w$ 
(Figure \ref{fig:feedback-loops-fos}).
We plot this time series for both FOS in our dataset.
For \textit{No Prop}, the fraction of papers recommended 
from the same FOS increases over time for both FOS (Figure~\ref{fig:feedback-loops-fos-no-propensity}).
This demonstrates the existence of a feedback loop 
that worsens exposure bias and reduces 
the number of papers recommended 
from a different FOS over time.
On the other hand, when we train our models using $\widehat{R}_w$, 
the feedback loop no longer exists 
and the fraction of papers recommended 
from a different FOS 
remains stable over time
(Figure~\ref{fig:feedback-loops-fos-with-propensity}). 
This shows that our proposed estimator continues 
to recommend relevant papers from a different FOS 
and corrects the feedback loop.

\section{Conclusion}
\label{sec:conclusion}
Proposing three estimators to correct for exposure bias,
we derive sufficient conditions for when 
they exhibit lower bias than the naive estimator 
and incorporate them into a learning procedure.
Theoretically, we prove that feedback loops can worsen exposure bias. 
Empirically, we show that proposed estimators 
improve performance against
the true
link probabilities,
leading to better estimates of true risk,
and combating feedback loops. 
Our methods can be extended to RSs 
that use different propensity or link probability models.
Using domain knowledge 
(e.g., through graphical models)
to improve propensity learning and
empirically evaluating our methods 
in other link recommendation tasks 
are promising future directions.
Exposure bias in link recommendation 
also raises fairness concerns. 
For example, in citation recommendation, 
certain authors or institutions 
might get unfair exposure 
which can be worsened by the RS.
Investigating exposure bias correction methods
for providing fairer recommendations 
would also be interesting future work.

\bibliography{refs}
\bibliographystyle{icml2021}

\clearpage

\appendix
\onecolumn

\section{Bias and Variance} \label{appendix:bias-variance}

\begin{lemma}\label{lemma:variance-bernoulli}
Let $X \sim \text{Bernoulli}(\theta)$ and $Y = aX + b(1-X)$, where $a$ and $b$ are some constants. Then
\begin{align*}
    \Var(Y) = \theta (1 - \theta) (a - b)^2.
\end{align*}
\end{lemma}

\textbf{Proof of Lemma \ref{lemma:bias-var-naive}}

\begin{lemma*}
The bias and variance of $\widehat{R}_{\text{naive}}(\widehat{o})$ are
\begin{align*}
    \Bias(\widehat{R}_{\text{naive}}) &= \left|\E[\widehat{R}_{\text{naive}}] - R(\widehat{o})\right| \\
        &= \frac{\Delta}{|\Da|} \left| \sum_{(i,j) \in \Da} y_{ij} (1 - \pi_{ij}) (1 - 2 \widehat{o}_{ij}) \right|,\\
    \Var(\widehat{R}_{\text{naive}}) &= \frac{\Delta^2}{|\Da|^2} \sum_{(i,j) \in \Da} y_{ij}\pi_{ij} (1 - y_{ij}\pi_{ij}).
\end{align*}
\end{lemma*}
\begin{proof}
We have
\begin{align*}
    \widehat{R}_{\text{naive}}(o, \widehat{y}) &= \frac{1}{|\Da|} \sum_{(i,j) \in \Da} \delta(o_{ij}, \widehat{o}_{ij}) \\
        &= \frac{1}{|\Da|} \sum_{(i,j) \in \Da} \left[ o_{ij} \delta(1, \widehat{o}_{ij}) + (1-o_{ij}) \delta(0, \widehat{o}_{ij}) \right] \\
    \therefore\,\, \E_o[\widehat{R}_{\text{naive}}(o, \widehat{y})] &= \frac{1}{|\Da|} \sum_{(i,j) \in \Da} \left[ y_{ij} \pi_{ij} \delta(1, \widehat{o}_{ij}) + (1-y_{ij} \pi_{ij}) \delta(0, \widehat{o}_{ij}) \right] \\
    &= \frac{1}{|\Da|} \sum_{(i,j) \in \Da} \left[ y_{ij} \pi_{ij} (1 - \widehat{o}_{ij})\delta(1, 0) + (1-y_{ij} \pi_{ij}) \widehat{o}_{ij} \delta(0, 1) \right] \\
    &= \frac{\Delta}{|\Da|} \sum_{(i,j) \in \Da} \left[ y_{ij} \pi_{ij} (1 - \widehat{o}_{ij}) + (1-y_{ij} \pi_{ij}) \widehat{o}_{ij} \right].
\end{align*}
The true risk is
\begin{align*}
    R(\widehat{y}) &= \frac{1}{|\Da|} \sum_{(i,j) \in \Da} \left[y_{ij} \delta(1, \widehat{o}_{ij}) + (1 - y_{ij}) \delta(0, \widehat{o}_{ij}) \right] \\
    &= \frac{1}{|\Da|} \sum_{(i,j) \in \Da} \left[y_{ij} (1-\widehat{o}_{ij}) \delta(1, 0) + (1 - y_{ij}) \widehat{o}_{ij} \delta(0, 1) \right] \\
    &= \frac{\Delta}{|\Da|} \sum_{(i,j) \in \Da} \left[y_{ij} (1-\widehat{o}_{ij}) + (1 - y_{ij}) \widehat{o}_{ij} \right].
\end{align*}
Thus the bias is
\begin{align*}
    \Bias(\widehat{R}_{\text{naive}}) &= \left|\E[\widehat{R}_{\text{naive}}] - R(\widehat{o})\right| \\
    &= \left| \frac{\Delta}{|\Da|} \sum_{(i,j) \in \Da} \left[ y_{ij} \pi_{ij} (1 - \widehat{o}_{ij}) + (1-y_{ij} \pi_{ij}) \widehat{o}_{ij} - y_{ij} (1-\widehat{o}_{ij}) - (1 - y_{ij}) \widehat{o}_{ij} \right]  \right| \\
    &= \frac{\Delta}{|\Da|} \left| \sum_{(i,j) \in \Da} y_{ij} (1 - \pi_{ij}) (1 - 2 \widehat{o}_{ij}) \right|.
\end{align*}
The variance is
\begin{align*}
    \Var(\widehat{R}_{\text{naive}}) &= \Var\left( \frac{1}{|\Da|} \sum_{(i,j) \in \Da} \left[ o_{ij} \delta(1, \widehat{o}_{ij}) + (1-o_{ij}) \delta(0, \widehat{o}_{ij}) \right] \right) \\
    &= \frac{1}{|\Da|^2} \sum_{(i,j) \in \Da} \Var\left(  o_{ij} \delta(1, \widehat{o}_{ij}) + (1-o_{ij}) \delta(0, \widehat{o}_{ij}) \right) \\
    &= \frac{1}{|\Da|^2} \sum_{(i,j) \in \Da} y_{ij} \pi_{ij} (1 - y_{ij} \pi_{ij}) \left(\delta(1, \widehat{o}_{ij}) - \delta(0, \widehat{o}_{ij}) \right)^2 \,\,\, \text{(using Lemma \ref{lemma:variance-bernoulli})} \\
    &= \frac{\Delta^2}{|\Da|^2} \sum_{(i,j) \in \Da} y_{ij} \pi_{ij} (1 - y_{ij} \pi_{ij}).
\end{align*}
\end{proof}
Lemmas \ref{lemma:rw_bias_var}, \ref{lemma:bias-var-pu}, and \ref{lemma:bias-var-ap} can be proved similarly.
\newline

\textbf{Proof of Theorem \ref{thm:comp_var}}

\begin{theorem*}[\textbf{Comparison of Variances}]
For all values of $\widehat{\pi}, \widehat{y}$, we have $\Var(\widehat{R}_{\text{AP}}) < \Var(\widehat{R}_{\text{naive}}), \,\, \text{and} \\
\Var(\widehat{R}_{\text{AP}}) < \Var(\widehat{R}_w) < \Var(\widehat{R}_{\text{PU}})$
\end{theorem*}
\begin{proof}
First we show that $\Var(\widehat{R}_{\text{AP}}) < \Var(\widehat{R}_{\text{naive}})$. We have
\begin{align*}
    \Var(\widehat{R}_{\text{AP}}) &= \frac{\Delta^2}{|\Da|^2} \sum_{(i,j) \in \Da} y_{ij}\pi_{ij} (1 - y_{ij}\pi_{ij}) \psi^2_{ij}, \\
    \text{where} \,\, & \psi_{ij} = \frac{1-\widehat{y}_{ij}}{1-\widehat{\pi}_{ij}\widehat{y}_{ij}} < 1, \\
    \Var(\widehat{R}_{\text{naive}}) &= \frac{\Delta^2}{|\Da|^2} \sum_{(i,j) \in \Da} y_{ij}\pi_{ij} (1 - y_{ij}\pi_{ij}).
\end{align*}
Using the fact that $\psi_{ij}^2 < 1 \,\, \forall (i,j) \in \Da$, we get $\Var(\widehat{R}_{\text{AP}}) < \Var(\widehat{R}_{\text{naive}})$.

Next, we show that $\Var(\widehat{R}_w) < \Var(\widehat{R}_{\text{PU}})$:
\begin{align*}
    \Var(\widehat{R}_w) &= \frac{\Delta^2}{|\Da|^2} \sum_{(i,j) \in \Da} y_{ij}\pi_{ij} (1 - y_{ij}\pi_{ij}) \left( \frac{1 - \widehat{o}_{ij}}{\widehat{\pi}^2_{ij}} + \widehat{o}_{ij} \psi^2_{ij} \right) \\
    \Var(\widehat{R}_{\text{PU}}) &= \frac{\Delta^2}{|\Da|^2} \sum_{(i,j) \in \Da} \frac{y_{ij}\pi_{ij} (1 - y_{ij}\pi_{ij})}{\widehat{\pi}^2_{ij}} \\
        &= \frac{\Delta^2}{|\Da|^2} \sum_{(i,j) \in \Da} y_{ij}\pi_{ij} (1 - y_{ij}\pi_{ij}) \left( \frac{1 - \widehat{o}_{ij}}{\widehat{\pi}^2_{ij}} + \frac{\widehat{o}_{ij}}{\widehat{\pi}^2_{ij}}  \right) \\
    \therefore\,\, \Var(\widehat{R}_w) - \Var(\widehat{R}_{\text{PU}}) &= \frac{\Delta^2}{|\Da|^2} \sum_{(i,j) \in \Da} y_{ij}\pi_{ij} (1 - y_{ij}\pi_{ij}) \widehat{o}_{ij} \left( \psi^2_{ij} - \frac{1}{\widehat{\pi}^2_{ij}} \right) \\
    \therefore\,\, \Var(\widehat{R}_w) - \Var(\widehat{R}_{\text{PU}}) &< 0 \,\,\, \left(\text{because $\psi_{ij} < 1$ and $\frac{1}{\widehat{\pi}_{ij}} > 1$}\right) \\
    \therefore\,\, \Var(\widehat{R}_w) &< \Var(\widehat{R}_{\text{PU}}).
\end{align*}
Next, we show that $\Var(\widehat{R}_{\text{AP}}) < \Var(\widehat{R}_w)$:
\begin{align*}
    \Var(\widehat{R}_{\text{AP}}) - \Var(\widehat{R}_w) &= \frac{\Delta^2}{|\Da|^2} \sum_{(i,j) \in \Da} y_{ij}\pi_{ij} (1 - y_{ij}\pi_{ij}) (1 - \widehat{o}_{ij}) \left( \psi^2_{ij} - \frac{1}{\widehat{\pi}^2_{ij}} \right) \\
    \therefore\,\, \Var(\widehat{R}_{\text{AP}}) - \Var(\widehat{R}_w) &< 0 \,\,\, \left(\text{because $\psi_{ij} < 1$ and $\frac{1}{\widehat{\pi}_{ij}} > 1$}\right) \\
    \therefore\,\, \Var(\widehat{R}_{\text{AP}}) &< \Var(\widehat{R}_w).
\end{align*}
\end{proof}

\textbf{Proof of Theorem \ref{thm:bias-comparisons}}

\begin{theorem*}[\textbf{Comparison of Biases}]
Under the bias approximations, a sufficient condition 
for $\Bias(\widehat{R}_w) = \Bias(\widehat{R}_{\text{PU}}) < \Bias(\widehat{R}_{\text{naive}})$ is
\begin{align*}
    \frac{\pi_{ij}}{2 - \pi_{ij}} < \widehat{\pi}_{ij} < 1, \,\, \forall (i, j) \in \Da,
\end{align*}
and for $\Bias(\widehat{R}_{\text{AP}}) < \Bias(\widehat{R}_{\text{naive}})$ is
\begin{align*}
    & \frac{\pi_{ij}}{2 - \pi_{ij}} < \widehat{\pi}_{ij} < 1 \,\, \text{and} \,\, 0 < \widehat{y}_{ij} < c y_{ij}, \,\, \forall (i,j) \in \Da \\
    & \text{where}\,\, c = \frac{2 (1 - \pi_{ij})}{1 - \widehat{\pi}_{ij} - \pi_{ij} y_{ij} + (2 - \pi_{ij}) \widehat{\pi}_{ij} y_{ij}} \geq 1.
\end{align*}
\end{theorem*}
\begin{proof}
We first derive the sufficient condition for $\Bias(\widehat{R}_w) = \Bias(\widehat{R}_{\text{PU}}) < \Bias(\widehat{R}_{\text{naive}})$. We have
\begin{align*}
    \Bias(\widehat{R}_{\text{naive}}) &\approx \frac{\Delta}{|\Da|} \sum_{(i,j) \in \Da'} y_{ij} (1 - \pi_{ij}), \\
    \Bias(\widehat{R}_w) \approx \Bias(\widehat{R}_{\text{PU}}) &\approx \begin{aligned}[t]
        & \frac{\Delta}{|\Da|} \left| \sum_{(i,j) \in \Da'} y_{ij}\left(1 - \frac{\pi_{ij}}{\widehat{\pi}_{ij}} \right) \right|.
        \end{aligned}
\end{align*}
If $1 > \widehat{\pi}_{ij} > \pi_{ij} \forall (i,j) \in \Da$, we have
\begin{align*}
    & \left(1 - \frac{\pi_{ij}}{\widehat{\pi}_{ij}} \right) > 0 \,\,\, \forall (i,j) \in \Da \\
    \therefore\,\, & \Bias(\widehat{R}_w) \approx \Bias(\widehat{R}_{\text{PU}}) \approx \begin{aligned}[t]
        & \frac{\Delta}{|\Da|} \sum_{(i,j) \in \Da'} y_{ij}\left(1 - \frac{\pi_{ij}}{\widehat{\pi}_{ij}} \right).
        \end{aligned} \\
    \therefore\,\, & \begin{aligned}[t]
    \Bias(\widehat{R}_w) - \Bias(\widehat{R}_{\text{naive}}) &= \frac{\Delta}{|\Da|} \sum_{(i,j) \in \Da'} y_{ij} \left( \pi_{ij} - \frac{\pi_{ij}}{\widehat{\pi}_{ij}} \right) \\
        &= \frac{\Delta}{|\Da|} \sum_{(i,j) \in \Da'} y_{ij} \pi_{ij} \left( 1 - \frac{1}{\widehat{\pi}_{ij}} \right) \\
        &< 0 \,\,\, \text{(because $\widehat{\pi}_{ij} < 1$)}
    \end{aligned} \\
    \therefore\,\, & \Bias(\widehat{R}_w) < \Bias(\widehat{R}_{\text{naive}}).
\end{align*}
If $0 < \widehat{\pi}_{ij} \leq \pi_{ij} \forall (i,j) \in \Da$, we have
\begin{align*}
    & \left(1 - \frac{\pi_{ij}}{\widehat{\pi}_{ij}} \right) \leq 0 \,\,\, \forall (i,j) \in \Da \\
    \therefore\,\, & \Bias(\widehat{R}_w) \approx \Bias(\widehat{R}_{\text{PU}}) \approx \begin{aligned}[t]
        & \frac{\Delta}{|\Da|} \sum_{(i,j) \in \Da'} y_{ij}\left( \frac{\pi_{ij}}{\widehat{\pi}_{ij}} - 1 \right).
        \end{aligned}.
\end{align*}
Then, a sufficient condition for $\Bias(\widehat{R}_w) = \Bias(\widehat{R}_{\text{PU}}) < \Bias(\widehat{R}_{\text{naive}})$ is
\begin{align*}
    y_{ij}\left( \frac{\pi_{ij}}{\widehat{\pi}_{ij}} - 1 \right) &< y_{ij} (1 - \pi_{ij}) \,\,\, \forall (i,j) \in \Da \\
    \therefore\,\, y_{ij}\left( \frac{\pi_{ij}}{\widehat{\pi}_{ij}} - 1 \right) &< y_{ij} (1 - \pi_{ij}) \,\,\, \forall (i,j) \in \Da \\
    \therefore\,\, \widehat{\pi}_{ij} &> \frac{2}{2 - \pi_{ij}} \,\,\, \forall (i,j) \in \Da.
\end{align*}
Next, we derive the sufficient condition for $\widehat{R}_{\text{AP}} < \widehat{R}_{\text{naive}}$. Observe that
\begin{align*}
    & \frac{\pi_{ij}}{2 - \pi_{ij}} < \widehat{\pi}_{ij} < 1 \,\,\, \forall (i, j) \in \Da \\
    \therefore\,\, & (1 -\pi_{ij}) y_{ij} - (1-\pi_{ij}y_{ij}) \tau_{ij} \geq 0 \,\,\, \forall (i, j) \in \Da \\
    \therefore\,\, & \widehat{R}_{\text{AP}} \approx \frac{\Delta}{|\Da|} \sum_{(i,j) \in \Da'} \left[ (1-\pi_{ij}) y_{ij} - (1-\pi_{ij}y_{ij}) \tau_{ij} \right], \,\, \text{where} \,\, \tau_{ij} = \left( \frac{\widehat{y}_{ij}(1 - \widehat{\pi}_{ij})}{1 - \widehat{\pi}_{ij} \widehat{y}_{ij}} \right).
\end{align*}
Therefore, when $\frac{\pi_{ij}}{2 - \pi_{ij}} < \widehat{\pi}_{ij} < 1 \,\,\, \forall (i, j) \in \Da$, a sufficient condition for $\widehat{R}_{\text{AP}} < \widehat{R}_{\text{naive}}$ is
\begin{align*}
    (1 -\pi_{ij}) y_{ij} - (1-\pi_{ij}y_{ij}) \tau_{ij} < y_{ij} (1 - \pi_{ij})  \,\,\, \forall (i, j) \in \Da \\
    \therefore\,\, 0 < \widehat{y}_{ij} < \left( \frac{2 (1 - \pi_{ij})}{1 - \widehat{\pi}_{ij} - \pi_{ij} y_{ij} + (2 - \pi_{ij}) \widehat{\pi}_{ij} y_{ij}} \right) y_{ij} \,\,\, \forall (i, j) \in \Da.
\end{align*}
\end{proof}

\section{Generalization Bound} \label{appendix:generalization-bound}

\textbf{Proof of Theorem \ref{thm:generalization-bound}}

\begin{theorem*}[\textbf{Generalization Bound}]
Let $\mathcal{F}$ be a class of functions 
$(\widehat{\pi}, \widehat{y})$.
Let $\delta(o_{ij}, \widehat{y}_{ij}) \leq \eta\,\,\forall (i,j) \in \Da$ 
and $\widehat{\pi}_{ij} \geq \epsilon > 0\,\,\forall (i,j) \in \Da$.
Then, for $\widehat{R} \in \left\{ \widehat{R}_w,  \widehat{R}_{\text{PU}}, \widehat{R}_{\text{AP}} \right\}$, with probability at least $1 - \delta$, we have
\begin{align}
    R(\widehat{y}) &\leq \widehat{R}(\widehat{y}, \widehat{\pi}) + B(\widehat{R}) + 2 \mathcal{G}(\mathcal{F}, \widehat{R}) + M \label{eq:rademacher-bound} \\
        &\leq \widehat{R}(\widehat{y}, \widehat{\pi}) + B(\widehat{R}_w) + 2 \widehat{\mathcal{G}}(\mathcal{F}, \widehat{R}_w) + 3 M, \label{eq:rademacher-empirical-bound}
\end{align}
where $M = \sqrt{\frac{4 \eta^2}{\epsilon^2 |\Da|} \log(\frac{2}{\delta})}$ and $B(\widehat{R})$ is the bias of $\widehat{R}$ derived in Section~\ref{sec:estimating-risk}.
\end{theorem*}
\begin{proof}
We proceed similarly to the standard Rademacher complexity generalization bound
proof \citep{shalev2014understanding}[Ch. 26]. Observe that
\begin{align}
    R(\widehat{y}) &= R(\widehat{y}) - \E_o[\widehat{R}(o, \widehat{y}, \widehat{\pi})] + \E_o[\widehat{R}(o, \widehat{y}, \widehat{\pi})] \nonumber \\
        &\leq B(\widehat{R}) + \E_o[\widehat{R}(o, \widehat{y}, \widehat{\pi})]. \label{eq:rademacher-proof-1}
\end{align}
Let $\Phi(o) = \sup_{(\widehat{\pi}, \widehat{y}) \in \mathcal{F}} \left[ \E_o[\widehat{R}(o, \widehat{y}, \widehat{\pi})] - \widehat{R}(o, \widehat{y}, \widehat{\pi}) \right]$. Then
\begin{align}\label{eq:rademacher-proof-2}
     \E_o[\widehat{R}(o, \widehat{y}, \widehat{\pi})]| \leq \widehat{R}(o, \widehat{y}, \widehat{\pi}) + \Phi(o).
\end{align}
Now we upper bound $\Phi(o)$.
Since $\delta(o_{ij}, \widehat{y}_{ij}) \leq \eta\,\,\forall (i,j)$ 
and $\widehat{\pi}_{ij} \geq \epsilon > 0,\,\,\forall (i,j)$ and $\forall \, \widehat{R} \in \left\{ \widehat{R}_w,  \widehat{R}_{\text{PU}}, \widehat{R}_{\text{AP}} \right\}$,
we have
\begin{align*}
    |\Phi(o) - \Phi(\Tilde{o})| \leq \frac{2 \eta}{\epsilon},
\end{align*}
if $o$ and $\Tilde{o}$ differ in only one coordinate, i.e., 
$o_{ij} \neq \Tilde{o}_{ij}$ for some $(i,j) \in \Da$ 
and $o_{lm} = \Tilde{o}_{lm} \forall (l,m) \in \Da$ 
s.t. $(i, j) \neq (l,m)$. 
Using McDiarmid’s Inequality, with probability at least $1 - \delta$, 
we have
\begin{align} \label{eq:rademacher-proof-3}
    \Phi(o) \leq \E[\Phi(o)] + C.
\end{align}
Next, we upper bound $\E[\Phi(o)]$.
Let $\bar{o}$ be a ghost sample independently drawn having the same distribution as $o$. We have
\begin{align}
    \E[\Phi(o)] &= \E_o\left[ \sup_{(\widehat{\pi}, \widehat{y}) \in \mathcal{F}} \left[ \E_o[\widehat{R}(o, \widehat{y}, \widehat{\pi})] - \widehat{R}(o, \widehat{y}, \widehat{\pi}) \right] \right] \nonumber \\
    &= \E_o\left[ \sup_{(\widehat{\pi}, \widehat{y}) \in \mathcal{F}} \E_{\bar{o}}\left[\widehat{R}(\bar{o}, \widehat{y}, \widehat{\pi}) - \widehat{R}(o, \widehat{y}, \widehat{\pi}) \,\big| \, o \right] \right] \nonumber \\
    &= \E_o\left[ \sup_{(\widehat{\pi}, \widehat{y}) \in \mathcal{F}} \E_{\bar{o}}\left[ \frac{1}{|\Da|} \sum_{(i,j) \in \Da} r(\bar{o}_{ij}, \widehat{\pi}_{ij}, \widehat{y}_{ij}) - \frac{1}{|\Da|} \sum_{(i,j) \in \Da} r(o_{ij}, \widehat{\pi}_{ij}, \widehat{y}_{ij}) \,\Big| \, o \right] \right] \nonumber \\
    &\leq \E_{o, \bar{o}} \left[ \sup_{(\widehat{\pi}, \widehat{y}) \in \mathcal{F}} \left[ \frac{1}{|\Da|} \sum_{(i,j) \in \Da} r(\bar{o}_{ij}, \widehat{\pi}_{ij}, \widehat{y}_{ij}) - \frac{1}{|\Da|} \sum_{(i,j) \in \Da} r(o_{ij}, \widehat{\pi}_{ij}, \widehat{y}_{ij}) \right] \right] \,\, \text{(Jensen's Inequality)} \nonumber \\
    &= \E_{o, \bar{o}, \sigma} \left[ \sup_{(\widehat{\pi}, \widehat{y}) \in \mathcal{F}} \left[ \frac{1}{|\Da|} \sum_{(i,j) \in \Da} \sigma_{ij} r(\bar{o}_{ij}, \widehat{\pi}_{ij}, \widehat{y}_{ij}) - \frac{1}{|\Da|} \sum_{(i,j) \in \Da} \sigma_{ij} r(o_{ij}, \widehat{\pi}_{ij}, \widehat{y}_{ij}) \right] \right] \nonumber \\
    &= \E_{o, \bar{o}, \sigma} \left[ \sup_{(\widehat{\pi}, \widehat{y}) \in \mathcal{F}} \left[ \frac{1}{|\Da|} \sum_{(i,j) \in \Da} \sigma_{ij} r(\bar{o}_{ij}, \widehat{\pi}_{ij}, \widehat{y}_{ij}) + \frac{1}{|\Da|} \sum_{(i,j) \in \Da} \sigma_{ij} r(o_{ij}, \widehat{\pi}_{ij}, \widehat{y}_{ij}) \right] \right] \nonumber \\
    &\leq \E_{o, \bar{o}, \sigma} \left[ \sup_{(\widehat{\pi}, \widehat{y}) \in \mathcal{F}} \left[ \frac{1}{|\Da|} \sum_{(i,j) \in \Da} \sigma_{ij} r(\bar{o}_{ij}, \widehat{\pi}_{ij}, \widehat{y}_{ij}) \right] + \sup_{(\widehat{\pi}, \widehat{y}) \in \mathcal{F}} \left[ \frac{1}{|\Da|} \sum_{(i,j) \in \Da} \sigma_{ij} r(o_{ij}, \widehat{\pi}_{ij}, \widehat{y}_{ij}) \right] \right] \nonumber \\
    &= 2 \mathcal{G}(\mathcal{F}, \widehat{R}). \label{eq:rademacher-proof-4}
\end{align}
Combining Eqs. \ref{eq:rademacher-proof-1},
\ref{eq:rademacher-proof-2}, \ref{eq:rademacher-proof-3},
and \ref{eq:rademacher-proof-4}, we get Eq. \ref{eq:rademacher-bound}.
Another application of McDiarmid’s Inequality 
allows us to obtain Eq. \ref{eq:rademacher-empirical-bound}
from Eq. \ref{eq:rademacher-bound}.
\end{proof}

\section{Feedback Loops} \label{appendix:feedback-loops}

\begin{lemma}[\textbf{Binomial Tail Bound}]\label{lemma:appendix-binomial-tail}
If the random variable $X_n \sim \frac{1}{n} \text{Binomial}(n, \theta)$, then for $\epsilon > 0$, we have
\begin{align*}
    \P(|X_n - \theta| > \epsilon) \leq 2 \exp\left( -2n \epsilon^2 \right).
\end{align*}
\begin{proof}
Observe that $X_n \in [0, 1]$. Applying Hoeffding's inequality gives us the desired result.
\end{proof}
\end{lemma}

\begin{lemma}\label{lemma:appendix-thm-proof-lemma-1}
Let $n \in \mathbb{N}$ and $\kappa$ be a fixed $C - 1$ simplex such that $\kappa_v n \in \mathbb{N} \,\,\, \forall v \in [C]$. The random variable $\Tilde{q}_v \sim \frac{1}{\kappa_v n} \text{Binomial}(\kappa_v n, q_v)$, where $q_v \in (0, 1)$. Assume that $q_v > q_w$ if $v > w$. We denote as $\widehat{e}$ the following $C-1$ simplex:
\begin{align*}
    \widehat{e} = \frac{1}{Z} \left[ \kappa_1 \Tilde{q}_1, \kappa_2 \Tilde{q}_2, \hdots, \kappa_C \Tilde{q}_C \right], \,\,\, \text{where} \,\, Z = \sum_{i \in [C]} \kappa_i \Tilde{q}_i.
\end{align*}
Let $\widehat{e}_{vw} = \frac{\widehat{e}_v}{\widehat{e}_v + \widehat{e}_w} = \frac{\kappa_v \Tilde{q}_v}{\kappa_w \Tilde{q}_w + \kappa_w \Tilde{q}_w}$ and $\kappa_{vw} = \frac{\kappa_v}{\kappa_v + \kappa_w}$. Then for a constant $\rho_{vw}$ such that
\begin{align*}
    0 < \rho_{vw} < \frac{\kappa_v \kappa_w (q_v - q_w)}{q_v \kappa^2_v + (q_v+q_w)\kappa_v \kappa_w + q_w \kappa^2_w},
\end{align*}
we have
\begin{align*}
    & |\Tilde{q}_v - q_v| < \epsilon_{vw}, \,\, |\Tilde{q}_w - q_w| < \epsilon_{vw} \implies \widehat{e}_{vw} - \kappa_{vw} > \rho_{vw}, \\
    & \text{for some constant $\epsilon_{vw}$ s.t.} \,\, 0 < \epsilon_{vw} < \frac{\rho_{vw} q_v \kappa^2_v - \kappa_v \kappa_w (q_w - q_v) + q_w \rho_{vw} \kappa_v (\kappa_v - \kappa_w)}{\rho_{vw} (\kappa^2_v - \kappa^2_w ) - 2 \kappa_v \kappa_w }.
\end{align*}
This is saying that, for $(v, w)$ s.t. $v > w$ the simplex $\widehat{e}$ will be more skewed towards $v$ than the simplex $\kappa$ if the sampled $\Tilde{q}_v$ and $\Tilde{q}_w$ are close to their mean values $q_v$ and $q_w$, respectively.
\end{lemma}
\begin{proof}
Observe that if $|\Tilde{q}_v - q_v| < \epsilon_{vw}$ and $|\Tilde{q}_w - q_w| < \epsilon_{vw}$, then the lowest value that $\widehat{e}_{vw}$ can take is
\begin{align*}
    & \widehat{e}^{(\text{min})}_{vw} = \frac{\kappa_v (q_v - \epsilon_{vw})}{\kappa_v (q_v - \epsilon_{vw}) + \kappa_w (q_w + \epsilon_{vw})}, \,\,\, \text{and} \\
    & \widehat{e}^{(\text{min})}_{vw} - \kappa_{vw} > \rho_{vw} \implies \widehat{e}_{vw} - \kappa_{vw} > \rho_{vw}.
\end{align*}
Therefore, we have
\begin{align*}
    & \widehat{e}^{(\text{min})}_{vw} - \kappa_{vw} > \rho_{vw} \,\,\, \text{and} \,\,\, \epsilon_{vw} < q_w \\
    & \impliedby \underbrace{\frac{\kappa_v (q_v - \epsilon_{vw})}{\kappa_v (q_v - \epsilon_{vw}) + \kappa_w (q_w + \epsilon_{vw})} - \frac{\kappa_v}{\kappa_v + \kappa_w} > \rho_{vw}}_{(1)} \,\,\, \text{and} \,\,\, \rho_{vw} < \frac{\kappa_v \kappa_w (q_v - q_w)}{q_v \kappa^2_v + (q_v+q_w)\kappa_v \kappa_w + q_w \kappa^2_w}.
\end{align*}
The inequality (1) above can further be simplified as
\begin{align*}
    & \frac{\kappa_v (q_v - \epsilon_{vw})}{\kappa_v (q_v - \epsilon_{vw}) + \kappa_w (q_w + \epsilon_{vw})} - \frac{\kappa_v}{\kappa_v + \kappa_w} > \rho_{vw} \\
    \impliedby & \epsilon_{vw} < \frac{\rho_{vw} q_v \kappa^2_v - \kappa_v \kappa_w (q_w - q_v) + q_w \rho_{vw} \kappa_v (\kappa_v - \kappa_w)}{\rho_{vw} (\kappa^2_v - \kappa^2_w ) - 2 \kappa_v \kappa_w }.
\end{align*}
This completes the proof.
\end{proof}

\begin{lemma}\label{lemma:appendix-thm-proof-lemma-2}
Let $\alpha$ be a fixed $C-1$ simplex and $\widehat{e}$ be the following $G-1$ simplex, $\widehat{e} = \frac{1}{Z} [\alpha_1 \Tilde{q}_1, \alpha_2 \Tilde{q}_2, \hdots, \alpha_C \Tilde{q}_C]$, where $Z = \sum_{z \in [C]} \alpha_z \Tilde{q}_z$ and the vector $\kappa \sim \frac{1}{n} \text{Multinomial}(n, \widehat{e})$. Let $\widehat{e}_{vw} = \frac{\widehat{e}_v}{\widehat{e}_v + \widehat{e}_w} = \frac{\Tilde{q}_v}{\Tilde{q}_w + \Tilde{q}_w}$ and $\kappa_{vw} = \frac{\kappa_v}{\kappa_v + \kappa_w}$.

Assume that $|\Tilde{q}_z - q_z| < \epsilon \,\,\, \forall z \in [C]$ where $q_z \in (0, 1)$ are fixed. If $|\kappa_v - \widehat{e}_v < \frac{\eta_{nw}}{C}|$ and $|\kappa_w - \widehat{e}_w < \frac{\eta_{nw}}{C}|$, then for some constant $\rho$, we have
\begin{align*}
    \widehat{e}_{vw} - \kappa_{vw} < \rho, \,\,\,
    \text{when} \,\,\, \eta_{vw} < \rho \left( \frac{q_v + q_w}{\max_{z \in [C]}q_z + \epsilon} \right).
\end{align*}
\begin{proof}
If $|\kappa_v - \widehat{e}_v < \frac{\eta_{nw}}{C}|$ and $|\kappa_w - \widehat{e}_w < \frac{\eta_{nw}}{C}|$, then the smallest value that $\kappa_{vw}$ can achieve is
\begin{align*}
    \kappa^{\text{(min)}}_{vw} = \frac{\widehat{e}_v - \frac{\eta_{nw}}{C}}{\widehat{e}_v + \widehat{e}_w}.
\end{align*}
This means that
\begin{align*}
    & \widehat{e}_{vw} - \kappa_{vw} < \rho \\
    \impliedby & \widehat{e}_{vw} - \kappa^{\text{(min)}}_{vw} < \rho \\
    \iff & \frac{\frac{\eta_{nw}}{C}}{\widehat{e}_v + \widehat{e}_w} < \rho \\
    \iff & \frac{\eta_{nw}}{C} < \rho (\widehat{e}_v + \widehat{e}_w).
\end{align*}
Since $|\Tilde{q}_z - q_z| < \epsilon \,\,\, \forall z \in [C]$, we have
\begin{align*}
    \widehat{e}_v &= \frac{\alpha_v \Tilde{q}_v}{\sum_{z \in [C]} \alpha_z \Tilde{q}_z} \\
        &> \frac{\alpha_v (q_v - \epsilon)}{\alpha_v (q_v - \epsilon) + \sum_{z \in [C], z \neq v} \alpha_z (q_z + \epsilon)} \\
        &> \frac{\alpha_v (q_v - \epsilon)}{C \max_{z \in [C]} \alpha_z (q_z + \epsilon)}, \\
    \text{and similarly} \,\,\, \widehat{e}_w &> \frac{\alpha_w (q_w - \epsilon)}{C \max_{z \in [C]} \alpha_w (q_z + \epsilon)} \\
    \therefore\,\, \widehat{e}_v + \widehat{e}_w &> \frac{ \alpha_v (q_v - \epsilon) + \alpha_w (q_w - \epsilon) }{C \max_{z \in [C]} (q_z + \epsilon)}.
\end{align*}
Therefore, we can set $\eta_{vw}$ such that
\begin{align*}
    \frac{\eta_{nw}}{C} &< \rho \left(\frac{\alpha_v (q_v - \epsilon) + \alpha_w (q_w - \epsilon) }{C \max_{z \in [C]} (q_z + \epsilon)} \right) \\
    \therefore\,\, \eta_{nw} &< \rho \left( \frac{\alpha_v (q_v - \epsilon) + \alpha_w (q_w - \epsilon) }{\max_{z \in [C]} q_z + \epsilon} \right).
\end{align*}
\end{proof}
\end{lemma}

\textbf{Proof of Theorem \ref{thm:feedback-loop-finite-sample}}

\begin{theorem*}
Suppose that $q_v > q_w$ if $v > w$. 
Let $\kappa^{(t)}_{vw} = \frac{\kappa^{(t)}_v}{\kappa^{(t)}_v + \kappa^{(t)}_w}$.
Let $A^{(t)}_{vw}$ represent the event 
that relative fraction of recommendations 
from $c_v$ to that from $c_w$ increases at time $t$,
i.e., $\kappa^{(t+1)}_{vw} > \kappa^{(t)}_{vw}$. 
Let $A^{(t)}$ be the event
that all relative fractions get skewed
towards $c_v$ from $c_w$ if $q_v > q_w$, 
i.e., $A^{(t)} = \bigcap_{(v,w) \in \mathcal{S}} A^{(t)}_{vw}$,
where $\mathcal{S} = \{ (v, w): v \in [C], w \in [C], v > w \}$.
Then, for constants $\epsilon, \eta > 0$ 
that only depend on $\kappa^{(t)}$ and $q$, we have
\begin{align*}
    & \P(A^{(t)} |\kappa^{(t)}) \begin{aligned}[t]
        & \geq 1 - 2C \exp\left( -2 n \left[ \epsilon^2 + \frac{\eta^2}{C^2} \right]  \right) \\
        & \geq 1 - 2C \exp\left( - \mathcal{O}\left(\frac{n}{C^2}\right) \right).
    \end{aligned}
\end{align*}
\end{theorem*}
\begin{proof}
We know that the estimated probabilities $\widehat{q}^{(t)}_v$ have distribution $\widehat{q}^{(t)}_v | \kappa^{(t)} \sim \frac{1}{n} \text{Binomial}(n \kappa^{(t)}_v, q_v)$. The simplex with normalized probabilities is $\widehat{e}^{(t+1)} = \frac{1}{Z} [\widehat{q}^{(t)}_1, \widehat{q}^{(t)}_2, \hdots, \widehat{q}^{(t)}_C]$, where $Z = \sum_{z \in [C]} \widehat{q}^{(t)}_z$.

Let $\Tilde{q}^{(t)}_v = \frac{\widehat{q}^{(t)}_v}{\kappa^{(t)}_v}$.
Observe that $\Tilde{q}^{(t)}_v | \kappa^{(t)} \sim \frac{1}{n \kappa^{(t)}_v} \text{Binomial}(n \kappa^{(t)}_v, q_v)$. We denote by $\widehat{e}^{(t+1)}_{vw}$,
\begin{align*}
    \widehat{e}^{(t+1)}_{vw} = \frac{\widehat{e}^{(t+1)}_v}{\widehat{e}^{(t+1)}_v + \widehat{e}^{(t+1)}_w} = \frac{\kappa^{(t)}_v \Tilde{q}^{(t)}_v}{\kappa^{(t)}_v \Tilde{q}^{(t)}_v + \kappa^{(t)}_w \Tilde{q}^{(t)}_w}.
\end{align*}
There are two main parts to the proof. 
First, we show that, with high probability, 
$\widehat{e}^{(t+1)}_{vw} - \kappa^{(t)}_{vw} > \rho \,\, \forall (v,w) \in \mathcal{S}$ for some constant $\rho$. 
Then, we show that, with high probability, 
$\widehat{e}^{(t+1)}_{vw} - \kappa^{(t+1)}_{vw} < \rho \,\, \forall (v,w) \in \mathcal{S}$. We combine these two results to show that, with high probability, $\kappa^{(t+1)}_{vw} > \kappa^{(t)}_{vw} \,\, \forall (v,w) \in \mathcal{S}$.

Using Lemma \ref{lemma:appendix-thm-proof-lemma-1}, for some $(v, w) \in \mathcal{S}$, we know that for some constant $\rho_{vw}$ such that
\begin{align*}
    0 < \rho_{vw} < \frac{\kappa^{(t)}_v \kappa^{(t)}_w (q_v - q_w)}{q_v (\kappa^{(t)}_v)^2 + (q_v+q_w)\kappa^{(t)}_v \kappa^{(t)}_w + q_w (\kappa^{(t)}_w)^2},
\end{align*}
we have
\begin{align*}
    & |\Tilde{q}^{(t)}_v - q_v| \leq \epsilon_{vw} \,\, \textit{and} \,\, |\Tilde{q}^{(t)}_w - q_w| \leq \epsilon_{vw} \implies \widehat{e}^{(t+1)}_{vw} - \kappa^{(t)}_{vw} \geq \rho_{vw}, \\
    & \text{for a constant $\epsilon_{vw}$ s.t.} \,\, 0 < \epsilon_{vw} < \frac{\rho_{vw} q_v (\kappa^{(t)}_v)^2 - \kappa^{(t)}_v \kappa^{(t)}_w (q_w - q_v) + q_w \rho_{vw} \kappa^{(t)}_v (\kappa^{(t)}_v - \kappa^{(t)}_w)}{\rho_{vw} ((\kappa^{(t)}_v)^2 - (\kappa^{(t)}_w)^2 ) - 2 \kappa^{(t)}_v \kappa^{(t)}_w } \\
    \implies & \P\left(\widehat{e}^{(t+1)}_{vw} - \kappa^{(t)}_{vw} \geq \rho_{vw}\right) \geq \P\left(|\Tilde{q}^{(t)}_v - q_v| \leq \epsilon_{vw}, |\Tilde{q}^{(t)}_w - q_w| \leq \epsilon_{vw}\right).
\end{align*}
Intuitively, this is saying 
that $\widehat{e}^{(t+1)}_{vw} - \kappa^{(t)}_{vw} > \rho_{vw}$ if $\Tilde{q}^{(t)}_v$ and $\Tilde{q}^{(t)}_w$ are close to $q_v$ and $q_w$, respectively.
Let $\rho = \min_{(v,w) \in \mathcal{S}} \rho_{vw}$ and $\epsilon = \min_{(v,w) \in \mathcal{S}} \epsilon_{vw}$. Then we have
\begin{align}
    \P \left( \bigcap_{(v,w) \in \mathcal{S}} \widehat{e}^{(t+1)}_{vw} - \kappa^{(t)}_{vw} \geq \rho \right) &\geq \P \left( \bigcap_{z \in [C]} |\Tilde{q}^{(t)}_z - q_z| \leq \epsilon \right) \\
    &= 1 - \P \left( \bigcup_{z \in [C]} |\Tilde{q}^{(t)}_z - q_z| \geq \epsilon \right) \\
    &\geq 1 - \sum_{z=1}^{C} \P \left( |\Tilde{q}^{(t)}_z - q_z| \geq \epsilon \right) \,\,\, \text{(Union Bound)} \\
    &\geq 1 - \sum_{z=1}^{C} 2 \exp\left( -2n \epsilon^2 \right)
    \,\,\, \text{(using Lemma \ref{lemma:appendix-binomial-tail})} \nonumber \\
    &= 1 - 2C \exp\left( -2n \epsilon^2 \right). \label{eq:feedback-finite-sample-proof-1}
\end{align}
Now, we show that $\widehat{e}^{(t+1)}_{vw}$ 
is close to $\kappa^{(t+1)}_{vw}$.
We know that $\kappa^{(t+1)} \sim \frac{1}{n} \text{Multinomial}(n, \widehat{e}^{(t+1)})$.
Let the event $Q^{(t)} = \bigcap_{z \in [C]} |\Tilde{q}^{(t)}_z - q_z| \leq \epsilon$.
Using Lemma \ref{lemma:appendix-thm-proof-lemma-2}, we know that, under $Q^{(t)}$, for some constant $\eta_{vw}$, we have 
\begin{align*}
    & \left|\widehat{e}^{(t+1)}_v - \kappa^{(t+1)}_v \right| < \frac{\eta_{vw}}{C} \,\,\, \text{and} \,\,\, \left|\widehat{e}^{(t+1)}_w - \kappa^{(t+1)}_w \right| < \frac{\eta_{vw}}{C} \implies \widehat{e}^{(t+1)}_{vw} - \kappa^{(t+1)}_{vw} < \rho, \\
    & \text{where} \,\, 0 < \eta_{vw} < \frac{ \kappa^{(t+1)}_v (q_v - \epsilon) + \kappa^{(t+1)}_w (q_w - \epsilon)}{\max_{z \in [C]} \kappa^{(t+1)}_z (q_z + \epsilon)} \\
    \implies & \P \left( \widehat{e}^{(t+1)}_{vw} - \kappa^{(t+1)}_{vw} < \rho \, \big| \, Q^{(t)} \right) \geq \P\left(\left|\widehat{e}^{(t+1)}_v - \kappa^{(t+1)}_v \right| < \frac{\eta_{vw}}{C}, \,\,\, \left|\widehat{e}^{(t+1)}_w - \kappa^{(t+1)}_w \right| < \frac{\eta_{vw}}{C} \right).
\end{align*}
Intuitively, this is saying that $\widehat{e}^{(t+1)}_{vw} - \kappa^{(t+1)}_{vw} < \rho$ if $\kappa^{(t+1)}_v$ and $\kappa^{(t+1)}_w$ are close to $\widehat{e}^{(t+1)}_v$ and $\widehat{e}^{(t+1)}_w$, respectively. Thus, for $\eta = \min_{(v,w) \in \mathcal{S}} \eta_{vw}$, we have
\begin{align}
    \P \left( \bigcap_{(v,w) \in \mathcal{S}} \widehat{e}^{(t+1)}_{vw} - \kappa^{(t+1)}_{vw} \leq \rho \, \Big| \, Q^{(t)} \right) &\geq \P\left( \bigcap_{z \in [C]} |\widehat{e}^{(t+1)}_z - \kappa^{(t+1)}_z| \leq \frac{\eta}{C} \right) \nonumber \\
    &= 1 - \P\left( \bigcup_{z \in [C]} |\widehat{e}^{(t+1)}_z - \kappa^{(t+1)}_z| \geq \frac{\eta}{C} \right) \nonumber \\  
    &\geq 1 - \sum_{z=1}^{C} \P\left(|\widehat{e}^{(t+1)}_z - \kappa^{(t+1)}_z| > \frac{\eta}{C} \right) \,\,\, \text{(Union Bound)} \nonumber \\
    &\geq 1 - 2C \exp\left( -\frac{2n \eta^2}{C^2} \right) \,\,\, \text{(using Lemma \ref{lemma:appendix-binomial-tail})}. \label{eq:feedback-finite-sample-proof-2}
\end{align}
Combining Eq. \ref{eq:feedback-finite-sample-proof-1} and \ref{eq:feedback-finite-sample-proof-2}, 
we get the desired result as follows:
\begin{align*}
    \P \left( \bigcap_{(v,w) \in \mathcal{S}} A^{(t)}_{vw} \right) &= \P \left( \bigcap_{(v,w) \in \mathcal{S}} \kappa^{(t+1)}_{vw} > \kappa^{(t)}_{vw} \right) \\
    &\geq \P \left( \bigcap_{(v,w) \in \mathcal{S}} \left( \widehat{e}^{(t+1)}_{vw} - \kappa^{(t)}_{vw} \geq \rho, \,\,\, \widehat{e}^{(t+1)}_{vw} - \kappa^{(t+1)}_{vw} \leq \rho  \right) \right) \\
    &\geq \P\left( \bigcap_{z \in [C]} \left( |\widehat{e}^{(t+1)}_z - \kappa^{(t+1)}_z| \leq \frac{\eta}{C}, \,\,\, |\Tilde{q}^{(t)}_z - q_z| \leq \epsilon \right) \right) \\
    &= \P\left( \bigcap_{z \in [C]} |\widehat{e}^{(t+1)}_z - \kappa^{(t+1)}_z| \leq \frac{\eta}{C} \, \Big| \, Q^{(t)} \right) \P\left( \bigcap_{z \in [C]}  |\Tilde{q}^{(t)}_z - q_z| \leq \epsilon \right) \\
    &\geq \left( 1 - 2C \exp\left( -\frac{2n \eta^2}{C^2} \right) \right) \left( 1 - 2C \exp\left( -2n \epsilon^2 \right) \right) \\
    &\geq 1 - 2C \left[ \exp\left( -2n \epsilon^2 \right) + \exp\left( -\frac{2n \eta^2}{C^2} \right) \right] \\
    &\geq 1 - 2C \exp\left( - \mathcal{O}\left(\frac{n}{C^2}\right) \right).
\end{align*}
\end{proof}

\textbf{Proof of Theorem \ref{thm:worsen_rate}}

\begin{lemma}[Convergence in Probability]\label{lemma:appendix-convergence-in-prob}
Let $X_n, Y_n$, and $Z$ be random variables such that $X_n \overset{p}{\to} Y_n$ and $Y_n \overset{p}{\to} Z$, then $X_n \overset{p}{\to} Z$.
\begin{proof}
For any $\epsilon > 0$, we have
\begin{align*}
    \P(|X_n - Z| \geq \epsilon) &= \P(|X_n - Y_n + Y_n - Z| \geq \epsilon) \\
    &\leq \P(|X_n - Y_n| + |Y_n-Z| \geq \epsilon) \\
    &\leq \P\left(|X_n - Y_n| \geq \frac{\epsilon}{2}\right) + \P\left(|X_n - Y_n| \geq \frac{\epsilon}{2}\right) \\
    &= 0.
\end{align*}
Therefore, $X_n \overset{p}{\to} Z$.
\end{proof}
\end{lemma}

\begin{theorem*}
Suppose that $q_v > q_w$. As $n \rightarrow \infty$, $\kappa^{(t)}_{vw} \overset{p}{\rightarrow} 1 - \frac{1}{1 + c^t}$, where $c = \frac{q_v}{q_w}$.
\end{theorem*}
\begin{proof}
At time step $t$, the fraction of recommendations 
from each group is $\kappa_t$. 
From group $g_v$, the user cites papers 
according to probability $q_v$. 
Therefore, $\widehat{q}^{(t)}_v \overset{p}{\rightarrow} \kappa^{(t)}_v q_v$. 
And the normalized estimate is
$\widehat{e}^{(t+1)} = \frac{1}{S} [\kappa^{(t)}_1 q_1, \hdots, \kappa^{(t)}_C q_C]$, where $S = \sum_{z \in [C]} \kappa^{(t)}_z q_z$. 
Since $\kappa^{(t+1)} \sim \frac{1}{n} \text{Multinomial}(n, \widehat{e}^{(t+1)})$,
we have
\begin{align}
    \kappa^{(t+1)} &\overset{p}{\rightarrow} \widehat{e}^{(t+1)} \nonumber \\
    \frac{\kappa^{(t+1)}_v}{\kappa^{(t+1)}_w} &\overset{p}{\rightarrow} \frac{q_v \kappa^{(t)}_v}{q_w \kappa^{(t)}_w} \nonumber \\
    &= c \frac{ \kappa^{(t)}_v}{ \kappa^{(t)}_w}. \label{eq:appendix-feedback-asymptotic-proof}
\end{align}
We know that $\frac{ \kappa^{(1)}_v}{ \kappa^{(1)}_w} \overset{p}{\rightarrow} c$. 
Combining this with Eq. \ref{eq:appendix-feedback-asymptotic-proof} and using Lemma \ref{lemma:appendix-convergence-in-prob} recursively, we get
\begin{align*}
    \frac{\kappa^{(t)}_v}{\kappa^{(t)}_w} &\overset{p}{\rightarrow} c^t \\
    \therefore\,\, 1 - \frac{1}{1 + \frac{\kappa^{(t)}_v}{\kappa^{(t)}_w}} &\overset{p}{\rightarrow} 1 - \frac{1}{1 + c^t} \,\,\, \text{(Continuous mapping theorem)} \\
    \therefore\,\, \frac{\kappa^{(t)}_v}{\kappa^{(t)}_v + \kappa^{(t)}_w} &\overset{p}{\rightarrow} 1 - \frac{1}{1 + c^t} \\
    \therefore\,\, \kappa^{(t)}_{vw} &\overset{p}{\rightarrow} 1 - \frac{1}{1 + c^t}.
\end{align*}
\end{proof}

\section{Experiments} \label{appendix:experiments}

\begin{table}
\caption{The distribution of the FOS in the two real-world datasets.
}
\label{table:fos-distributions-in-dataset}
\begin{center}
\begin{sc}
\begin{tabular}{lcc}
\toprule
FOS & Dataset 1 & Dataset 2 \\
\midrule
Art & $0.03\%$ & $0.08\%$ \\
Biology & $26.48\%$ & $23.43\%$ \\
Business & $0.38\%$ & $0.10\%$ \\
Chemistry & $10.11\%$ & $15.67\%$ \\
Computer Science & $9.40\%$ & $3.42\%$ \\
Economics & $2.51\%$ & $0.03\%$ \\
Engineering & $6.24\%$ & $17.98\%$ \\
Environmental Science & $0.13\%$ & $0.03\%$ \\
Geography & $0.48\%$ & $0.40\%$ \\
Geology & $1.45\%$ & $0.46\%$ \\
History & $0.04\%$ & $0.03\%$ \\
Materials Science & $3.06\%$ & $19.09\%$ \\
Mathematics & $7.17\%$ & $1.03\%$ \\
Medicine & $21.28\%$ & $13.90\%$ \\
Philosophy & $0.03\%$ & $0.01\%$ \\
Physics & $2.99\%$ & $3.14\%$ \\
Political Science & $0.18\%$ & $0.01\%$ \\
Psychology & $7.49\%$ & $1.14\%$ \\
Sociology & $0.55\%$ & $0.05\%$ \\
\bottomrule
\end{tabular}
\end{sc}
\end{center}
\end{table}

Table \ref{table:fos-distributions-in-dataset} provides the distribution of the various FOS in both the datasets used for the real-world dataset experiments (Section~\ref{sec:real-world-datasets}). We can see that the FOS distributions are different. For example, Dataset 2 has substantially more \textit{Materials Science} and \textit{Engineering} papers.

\end{document}